\documentclass[format=acmsmall, review=false, screen=true]{acmart}

\usepackage{booktabs} 
\newtheorem{myRemark}{\textbf{Remark}}

\usepackage[ruled]{algorithm2e} 

\SetAlFnt{\small}
\SetAlCapFnt{\small}
\SetAlCapNameFnt{\small}
\SetAlCapHSkip{0pt}
\IncMargin{-\parindent}

\acmJournal{TIST}
\acmVolume{xx}
\acmNumber{xx}
\acmArticle{xx}
\acmYear{2018}
\acmMonth{10}
\copyrightyear{2018}

\setcopyright{acmcopyright}
\acmJournal{TIST}
\acmYear{2018} \acmVolume{1} \acmNumber{1} \acmArticle{1} \acmMonth{10} \acmPrice{15.00}\acmDOI{10.1145/3284174}

\received{xx 2017}
\received[revised]{xx 2018}
\received[accepted]{xx 2018}

\begin{document}
\title[]{Reconstruction of Hidden Representation for Robust Feature Extraction}
\titlenote{This work is supported by the National Science Foundation of China (Nos. 61773324, 61573292, 61572406).}

\author{Zeng Yu}
\orcid{}
\affiliation{%
  \institution{Southwest Jiaotong University}
  \department{School of Information Science and Technology, National Engineering Laboratory of Integrated Transportation Big Data Application Technology}
  \streetaddress{}
  \city{Chengdu}
  \postcode{611756}
  \country{China}}
\email{zyu7@gsu.edu}
\author{Tianrui Li}
\authornote{This is the corresponding author.}
\affiliation{%
  \institution{Southwest Jiaotong University}
    \department{School of Information Science and Technology, National Engineering Laboratory of Integrated Transportation Big Data Application Technology}
  \city{Chengdu}
  \postcode{611756}
  \country{China}
}
\email{trli@swjtu.edu.cn}
\author{Ning Yu}
\affiliation{%
 \institution{The College at Brockport State University of New York}
   \department{Department of Computing Sciences}
 \streetaddress{}
 \city{Brockport}
 \state{NY}
 \postcode{14420}
 \country{USA}
 }
\email{nyu@brockport.edu}
\author{Yi Pan}
\affiliation{%
  \institution{Georgia State University}
    \department{Department of Computer Science}
  \streetaddress{}
   \city{Atlanta}
   \postcode{30302}
  \state{GA}
  \country{USA}
}
\email{yipan@gsu.edu}
\author{Hongmei Chen}
\affiliation{%
  \institution{Southwest Jiaotong University}
    \department{School of Information Science and Technology, National Engineering Laboratory of Integrated Transportation Big Data Application Technology}
  \city{Chengdu}
  \postcode{611756}
  \country{China}
  }
\email{hmchen@swjtu.edu.cn}
\author{Bing Liu}
\affiliation{%
  \institution{University of Illinois at Chicago}
  \department{Department of Computer Science}
  \city{Chicago}
  \state{IL}
  \postcode{60607}
  \country{USA}
}
  \email{liub@cs.uic.edu}

\begin{abstract}
This paper aims to develop a new and robust approach to feature representation. Motivated by the success of Auto-Encoders, 
we first theoretically analyze and summarize the general properties of all algorithms that are based on traditional Auto-Encoders: 1) The reconstruction error of the input can not be lower than a lower bound, which can be viewed as a guiding principle for reconstructing the input. Additionally, when the input is corrupted with noises, the reconstruction error of the corrupted input also can not be lower than a lower bound. 2) The reconstruction of a hidden representation achieving its ideal situation is the necessary condition for the reconstruction of the input to reach the ideal state. 3) Minimizing the Frobenius norm of the Jacobian matrix of the hidden representation has a deficiency and may result in a much worse local optimum value. We believe that minimizing the reconstruction error of the hidden representation
is more robust than minimizing the Frobenius norm of the Jacobian matrix of the hidden representation. Based on the above analysis, we propose a new model termed {\em Double Denoising Auto-Encoders} (DDAEs), which uses corruption and reconstruction on both the input and the hidden representation. We demonstrate that the proposed model is highly flexible and extensible and has a potentially better capability to learn invariant and robust feature representations. We also show that our model is more robust than Denoising Auto-Encoders (DAEs) for dealing with noises or inessential features. Furthermore, we detail how to train DDAEs with two different pre-training methods by optimizing the objective function in a combined and separate manner, respectively. Comparative experiments illustrate that the proposed model is significantly better for representation learning than the state-of-the-art models.
\end{abstract}

%
%


%
%


\maketitle

\renewcommand{\shortauthors}{Z. Yu et al.}

\section{Introduction}
\label{Introduct}

Representation learning via deep neural networks has developed into an important area of machine learning research in recent years. This development has also witnessed a wide range of successful applications in the fields of computer vision \cite{krizhevsky2012imagenet}, speech recognition \cite{deng2010binary}, and natural language processing \cite{socher2011semi}. Reviews of recent progresses can be found in
\cite{bengio2009learning,bengio2013PAMI,bengio2013representation,LeCun2015Deeplearning,NingYu,Miotto2017Deep,Litjens2017A}.

A deep neural network usually has a deep architecture that uses at least one layer to learn the feature representation of the given data. A representation learning procedure is applied to discover multiple levels of representation: the higher the level, the more abstract the representation. It has been shown that the performance of deep neural networks is heavily dependent on the multilevel representation of the data \cite{LeCun2015Deeplearning}. In the past few years, researchers have endeavored to design a variety of efficient deep learning algorithms that may capture some characteristics of the data-generating distribution \cite{Zhuang2017Representation,Liao2016Graph,Yangzhang2017,Yang2015A}. Among these algorithms, the traditional Auto-Encoders (AEs) \cite{bengio2007greedy} perhaps received the most research attention due to their conceptual simplicity, ease of training, and inference and training efficiency. They are used to learn the data-generating distribution of the input data by minimizing the reconstruction error of the input $\sum_{x\in X} L(x,g(f(x)))$, where $f(.)$ is the encoder function, $g(.)$ is the decoder function and $L(.)$ is the reconstruction error. Recently, they have become one of the most promising approaches to representation learning for estimating the data-generating distribution. Since the appearance of Auto-Encoders, many variants of representation learning algorithms based on Auto-Encoders have been proposed, e.g., Sparse Auto-Encoders \cite{kavukcuoglu2009learning,sparseautoencoder2016}, Denoising Auto-Encoders (DAEs) \cite{vincent2008extracting}, Higher Order Contractive Auto-Encoders \cite{Rifai2011Higherorder}, Variational Auto-Encoders \cite{variationalAuto-encod}, Marginalized Denoising Auto-Encoders \cite{Chen2014Marginalized}, Generalized Denoising Auto-Encoders \cite{bengio2013generalized}, Generative Stochastic Networks \cite{Generative-Stochastic-Networks2014}, Masked Autoencoder
for Distribution Estimation (MADE) \cite{germain2015made}, Laplacian Auto-Encoders \cite{Jia2015Laplacian}, Adversarial Auto-Encoders \cite{AdversarialAutoencoders2015}, Ladder Variational Auto-Encoders \cite{Ladder-Variational2016} and so on.

In an Auto-Encoder-based algorithm, minimizing the reconstruction error of the input with the encoder and decoder functions is a common practice for feature learning. The learned features are usually applied in subsequent tasks such as supervised classification \cite{potential-autoencoder2015}. In the past few years, many research works have shown that the reconstruction of the input with the encoder and decoder functions is not only an efficient way for learning feature representation, but its resulting representations also substantially help the performance of the subsequent tasks. In general, a lower value of the reconstruction error of the input has a better feature representation of the input. In an ideal situation, the value of this reconstruction error is equal to 0, i.e., the input can be completely reconstructed. In this paper, we show that the reconstruction error of the input from every traditional Auto-Encoders based algorithm has a lower bound, which is greater than or equal to 0.

As an important method of representation learning, minimizing the Frobenius norm of the Jacobian matrix of the hidden representation has been widely used in deep learning models. The first application is the CAEs \cite{rifai2011contractive}, which try to learn locally invariant features by minimizing the Frobenius norm of the Jacobian matrix of hidden representation. After that, many frameworks based on minimizing the Frobenius norm of the Jacobian matrix of hidden representation have been developed in computer vision tasks. Specifically, Liu et al. \cite{Liu-2016} developed a multimodal feature learning model with stacked CAEs for video classification. To find stable features, Schulz et al. \cite{Schulz-2015} designed a two-layer encoder which is regularized by an extension of a previous work on CAEs. Geng et al. \cite{Geng-2017} proposed a novel deep supervised and contractive neural network for SAR image classification by using the idea of minimizing the Frobenius norm of the Jacobian matrix of hidden representation. Shao et al. \cite{Shao-2016} introduced an enhancement deep feature fusion method for rotating machinery fault diagnosis through a combination of DAEs and CAEs. However, we will demonstrate that minimizing the Frobenius norm of the Jacobian matrix of hidden representation has a deficiency in learning feature representation. 

To learn robust feature representation, minimizing the reconstruction error of hidden representation is also important and efficient. This idea has been emphasized by popular deep learning algorithms such as Ladder Networks \cite{ladder-networks2015,Ladder2-2016} and  Target Propagation Networks \cite{target-propagation2014}. In order to reconstruct the hidden representation, Ladder Networks need two streams of information to reconstruct the hidden representation: one is used to generate a clean hidden representation with an encoder function; the other is utilized to reconstruct the clean hidden representation with a combinator function \cite{ladder-networks2015,Ladder2-2016}. The final objective function is the sum of all the reconstruction errors of hidden representation. It should be noted that reconstructing the hidden representation in each layer needs to use information of two layers, which makes Ladder Networks difficult to be trained with a layer-wise pre-training strategy. Training a deep learning model in a layer-wise manner, as it is known, is an unsupervised learning approach, which may have many potential advantages. To reconstruct the hidden representation, Target Propagation Networks \cite{target-propagation2014} can be trained in a layer-wise manner. Nevertheless, in the Target Propagation Networks, reconstructing hidden representation is decomposed into two separate targets, which may be trapped into a local optimum. To the best of our knowledge, reconstructing hidden representation as a whole and training it in a layer-wise manner has not yet been investigated so far.

In this paper, we first study the general properties of all algorithms based on the traditional Auto-Encoders. We aim to design a robust approach for feature representation based on these properties. We follow the framework of layer-wise pre-training and consider the idea of reconstruction of hidden representation. We propose a new deep learning model that takes advantage of corruption and reconstruction. Our model consists of two separate parts: constraints on the input (Constraints Part) and reconstruction on the hidden representation (Reconstruction Part). Constraints Part can be viewed as a traditional deep learning model such as auto-encoder and its variants. Reconstruction Part can be viewed as explicitly regularizing the hidden representation or adding additional feedback to the pre-training stage. For simplicity and convenience, we use a DAE as the Constraints Part to build our model. Because the best results are obtained by utilizing the corruption in both input and hidden representation, we refer it as Double Denoising Auto-Encoders (DDAEs).

The contributions of this paper are summarized as follows:
\begin{itemize}
\item We prove that for all algorithms based on traditional Auto-Encoders, the reconstruction error of the input can not be lower than a lower bound, which can sever as a guiding principle for the reconstruction of the input. We also show that the necessary condition for the reconstruction of the input to reach the ideal state is that the reconstruction of hidden representation achieves its ideal condition. When the input is corrupted with noises, we demonstrate that the reconstruction error of the corrupted input also can not be lower than a lower bound.
\item We validate that minimizing the Frobenius norm of the Jacobian matrix of the hidden representation has a deficiency and may result in a much worse local optimum value. We also show that minimizing reconstruction error of the hidden representation for feature representation is more robust than minimizing the Frobenius norm of the Jacobian matrix, which may be the main reason why the proposed DDAEs always outperform CAEs.
\item We propose a new approach to learn robust feature representations of the input based on the above evidences. Compared with the existing methods, DDAEs have the following advantages: 1) DDAEs are flexible and extensible and have a potentially better capability of learning invariant and robust feature representations. 2) For dealing with noises or some inessential features, DDAEs are more robust than DAEs. 3) DDAEs can be trained with two different pre-training strategies by optimizing the objective function in a combining or separate manner, respectively.
\end{itemize}

The rest of this paper is organized as follows. Section \ref{Section2} introduces the basic DAEs and CAEs. Section \ref{Section3} presents the lower bound of the reconstruction error of the input and the necessary condition for the reconstruction of the input to reach its ideal state. Section \ref{Section4} illustrates the defect of CAEs and gives a theoretical proof on why DDAEs can outstrip CAEs. Section \ref{Section5} describes the proposed DDAEs framework. Section \ref{Section6} compares the performance of DDAEs with other relevant state-of-the-art representation learning algorithms using various testing datasets. Conclusions together with some further studies are summarized in the last section.

\section{Preliminaries}
\label{Section2}

DDAEs are designed according to the traditional Auto-Encoders \cite{bengio2007greedy} that learn feature representation by minimizing the reconstruction error. For ease of understanding, we reveal DDAEs by starting to describe some conventional auto-encoder variants and notations.

\subsection{Denoising Auto-encoders (DAEs): Extracting Robust Features of Reconstruction}

Similar to traditional Auto-Encoders \cite{bengio2007greedy}, the Denoising Auto-Encoders (DAEs) \cite{vincent2008extracting} firstly use the encoder and decoder procedures to train one-layer neural network by minimizing the reconstruction error, and then stack a deep neural network with the trained layers. The only difference between traditional Auto-Encoders and DAEs is that DAEs train the neural network with corrupted input while the traditional Auto-Encoders use the original input. The corrupted input $\tilde{x}\in\Re^{D_{x}}$ is usually obtained from a conditional distribution $q(\tilde{x}|x)$ by injecting some noises into the original input $x\in\Re^{D_{x}}$. Typically, the most widely-used noises in the simulations are Gaussian noise $\tilde{x}=x+\epsilon,\epsilon\sim\mathcal{N}(0,\sigma^{2}I)$ and masking noise, where $\nu\%$ ($\nu$ is given by researchers) of the input components are set to 0.

To extract robust features, a DAE firstly maps the corrupted input $\tilde{x}$ to a hidden representation $h\in\Re^{D_{h}}$ by the encoder function $f$:
\begin{equation}
h=f(\tilde{x})=S_{f}(\textbf{W}\tilde{x}+b_{h}),
\end{equation}
where $\textbf{W}\in\Re^{D_{h}\times D_{x}}$ is a connection weight matrix, $b_{h}\in\Re^{D_{h}}$ is a bias vector of hidden representation and $S_{f}$ is an activation function, typically a logistic $sigmoid(\tau)=\frac{1}{1+e^{-\tau}}$. After that, the DAE reversely maps the hidden representation $h$ back to a reconstruction input $x^{*}\in\Re^{D_{x}}$ through the decoder function $g$:
\begin{equation}
x^{*}=g(h)=S_{g}(\textbf{W}'h+b_{x}),
\end{equation}
where $\textbf{W}'\in\Re^{D_{x}\times D_{h}}$  is a tied weight matrix, i.e., $\textbf{W}'=\textbf{W}^{T}$, $b_{x}\in\Re^{D_{x}}$ is a bias vector and $S_{g}$ is an activation function, typically either the identity (yielding linear reconstruction) or a sigmoid. Finally, the DAE learns the robust features by minimizing the reconstruction error on a training set $X=\{x_{1},x_{2},\cdots,x_{N}\}$.
\begin{equation}
\mathcal{J}_{DAE}(\theta)= \sum_{x\in X}E \left[ L(x,g(f(\tilde{x})))\right ],
\end{equation}
where $\theta=\{\textbf{W},b_{x},b_{h}\}$, $E(\delta)$ is the mathematical expectation of $\delta$, $L$ is the reconstruction error. Typically the squared error $L(x,y)=\|x-y\|^{2}$ is used when $S_{g}$ is the identity function and the cross-entropy loss $L(x,y)=-\left[\sum_{i=1}^{D_{x}}x_{i}log(y_{i})+(1-x_{i})log(1-y_{i})\right]$ is selected when $S_{g}$ is the sigmoid function.

It has been shown that DAEs can extract robust features by injecting some noises into the original input and implicitly capture the data-generating distribution of input in the conditions that the reconstruction error is the squared error and the data are continuous-valued with Gaussian corruption noise \cite{alain2014regularized}, \cite{bengio2013generalized}, \cite{vincent2011connection}.

\subsection{Contractive Auto-encoders (CAEs): Extracting Locally Invariant Features of Hidden Representation}

To extract locally invariant features, the CAEs \cite{rifai2011contractive} penalize the sensitivity by adding an analytic contractive penalty to the traditional Auto-Encoders. The contractive penalty is the Frobenius norm of first derivatives $\|J_{f}(x)\|_{F}^{2}$ of the encoder function $f(x)$ with respect to the input $x$.

Formally, the objective optimized by a CAE is
\begin{equation}
\mathcal{J}_{CAE}(\theta)= \sum_{x\in X}\left [ L(x,g(f(x)))+\alpha\|J_{f}(x)\|_{F}^{2}\right],
\end{equation}
where $\alpha$ is a hyper parameter that controls the strength of the regularization. For a sigmoid encoder, the contractive penalty is simply computed:
\begin{equation}
\|J_{f}(x)\|_{F}^{2} = \sum_{j=1}^{D_{h}}(h_{j}(1-h_{j}))^{2}\sum_{i=1}^{D_{x}}W_{ij}^{2}.
\end{equation}
Compared with DAEs, the CAEs have at least two differences: 1) The penalty is analytic rather than stochastic; 2) A hyper parameter $\alpha$ allows to control the tradeoff between reconstruction and robustness. Actually, in an optimizing searching algorithm, it seems more likely that CAEs try to find invariant features by restricting step lengths to small numbers (that is, numbers close to zero) in each search.

\section{Lower Bound of the Reconstruction Error of the Input}
\label{Section3}

Generally, in an algorithm based on traditional Auto-Encoders, the smaller the reconstruction error of the input, the better the algorithm. Ideally, the value of reconstruction error of the input is equal to 0. It means that the algorithm can completely reconstruct the input. However, in this paper, we prove that the reconstruction error of the input has a lower bound, which can be viewed as a criterion for the reconstruction of the input. We also illustrate that the reconstruction of hidden representation achieves its ideal condition is the necessary condition for the reconstruction of the input to reach the ideal state. When the input is corrupted with noises, we demonstrate that the reconstruction error of the corrupted input has a lower bound, too.

\subsection{Lower Bound and Necessary Condition}

We present the lower bound of reconstruction error of the input and a rigorous theoretical analysis below. We also reveal the necessary condition for the reconstruction of the input to reach the ideal state.
\newtheorem{Theorem}{Theorem}
\begin{Theorem} \label{Theorem}
Let $L(x,y)=\|x-y\|^{2}$ be the squared error. If we use the clean input $x$ and clean hidden representation $h_{c}$ to reconstruct themselves, then as $x_{c}^{*} \rightarrow x$, we have
\begin{equation}
L(x,g(f(x))) \geq  L(h_{c},f(g(h_{c}))) /  \|J_{f}(x)\|_{F}^{2},
\end{equation}
where $h_{c}$ is the corresponding hidden representation of the clean input $x$, i.e., $h_{c}=f(x)$, $x_{c}^{*}=g(h_{c})=g(f(x))$ is the reconstructed input and $\|J_{f}(x)\|_{F}^{2}=0$ iff the encoder function $f(x)$ is a constant.

Furthermore, we can get that
\begin{equation}
\begin{split}
\mathcal{J}_{CAE}(\theta) & =  \sum\limits_{x\in X} \left[ L(x,g(f(x)))+\alpha\|J_{f}(x)\|_{F}^{2} \right] \\
& \geq \sum\limits_{h_{c}\in H_{c}} \lambda\sqrt{L(h_{c},f(g(h_{c})))}.
\end{split}
\end{equation}
where $\mathcal{J}_{CAE}(\theta)$ is the objective function of the CAE, $X=\{x_{1}, x_{2}, \cdots, x_{N}\}$ is a training set.
\end{Theorem}
\begin{proof}
For a clean input $x\in X$, the corresponding clean hidden representation and reconstructed input are $h_{c}$ and  $x_{c}^{*}$, respectively. Let $h_{c}^{*}=f(x_{c}^{*})$ be the reconstructed hidden representation. Then we can approximate the encoder function $f(x_{c}^{*})$ by its Taylor expansion around $x$ with Lagrange remainder term
\begin{equation}
f(x_{c}^{*})=f(x)+(x_{c}^{*}-x)^{T}\nabla f[x+\rho(x_{c}^{*}-x)],         \nonumber
\end{equation}
where $\nabla f[x+\rho(x_{c}^{*}-x)]$ is the first-order derivative of encoder function $f(\cdot)$ with respect to $x+\rho(x_{c}^{*}-x)$ and $\rho\in(0,1)$ is a constant.

Using the triangle inequality, we have that
\begin{equation}
\begin{aligned}
L(h_{c},f(g(h_{c})))&=\|h_{c}^{*}-h_{c} \|^{2} \\                \nonumber
&=  \|f(x_{c}^{*})-f(x)\|^{2} \\                                  \nonumber
&=  \|(x_{c}^{*}-x)^{T} \nabla f[x+\rho(x_{c}^{*}-x)] \|^{2}\\                          \nonumber
&\leq  \| x_{c}^{*}-x\|^{2} \cdot  \| \nabla f[x+\rho(x_{c}^{*}-x)] \|_{F}^{2},          \nonumber
\end{aligned}
\end{equation}
where $\|\cdot \|^{2}$ is the squared error and $\| A \|_{F}^{2}$ is the square of Frobenius norm on matrix $A$.

When $x_{c}^{*} \rightarrow x$, i.e., the reconstructed input $x_{c}^{*}$ infinitely approaches $x$, then as an ideal state, we have
\begin{equation}
\begin{aligned}
\lim_{x_{c}^{*}\to x} \| \nabla  f[x+\rho(x_{c}^{*}-x)] \|_{F}^{2} = \|J_{f}(x)\|_{F}^{2},
\end{aligned}
\end{equation}
and
\begin{equation}
\begin{split}
L(h_{c},f(g(h_{c}))) & \leq  \| x_{c}^{*}-x\|^{2}  \cdot \|J_{f}(x)\|_{F}^{2} \\
& = L(x,g(f(x))) \cdot \|J_{f}(x)\|_{F}^{2},
\end{split}
\label{Equ:main}
\end{equation}
Hence, we get
\begin{equation}
L(x,g(f(x))) \geq  L(h_{c},f(g(h_{c}))) /  \|J_{f}(x)\|_{F}^{2},   \nonumber
\end{equation}
where $\|J_{f}(x)\|_{F}^{2}=0$ if and only if the encoder function $f(x)$ is a constant.

Moreover, from (\ref{Equ:main}) and the basic inequality $2\sqrt{ab}\leq a+b,\ a,b\geq 0$, we have
\begin{flalign}
\begin{split}
\lambda \sqrt{L(h_{c},f(g(h_{c})))} &\leq \| x_{c}^{*}-x\|^{2} + \frac{\lambda^{2}}{4}  \|J_{f}(x)\|_{F}^{2} \\   \nonumber
& = L(x,g(f(x))) + \frac{\lambda^{2}}{4} \|J_{f}(x)\|_{F}^{2},
\end{split}
\end{flalign}
Hence, let $\alpha = \frac{\lambda^{2}}{4}$,  we can get that
\begin{equation}
\begin{split}
\mathcal{J}_{CAE}(\theta) & =\sum\limits_{x\in X} \left[ L(x,g(f(x)))+\alpha\|J_{f}(x)\|_{F}^{2} \right] \\             \nonumber
& \geq \sum\limits_{h_{c}\in H_{c}} \lambda\sqrt{L(h_{c},f(g(h_{c})))}.      \qedhere
\end{split}
\end{equation}
\end{proof}

\begin{myRemark}
Theorem 1 summarizes the general rule of the reconstruction of the input for all the algorithms based on traditional Auto-Encoders. As we can see, the reconstruction error of the input can not be lower than a lower bound, which gives a guiding principle for reconstructing the input.
\end{myRemark}

The traditional view for reconstructing the input is that the smaller the reconstruction error of the input, the better the algorithm. The ideal situation is that the value of the reconstruction error of the input is 0, i.e., the algorithm can completely reconstruct the input. However, Theorem 1 demonstrates that the ideal value of the reconstruction error of the input is a lower bound, which is greater than or equal to 0. Hence, compared with traditional view, Theorem 1 gives a more accurate quantitative description of the reconstruction error of the input.

\begin{myRemark}
Theorem 1 also provides a necessary condition for the reconstruction of the input to reach the ideal state, namely, the reconstruction of hidden representation achieves its ideal condition.
\end{myRemark}

Because if the reconstruction of hidden representation does not achieve its ideal condition, the reconstruction of the input also can not reach the ideal state. Nevertheless, when the reconstruction of hidden representation achieves the ideal state, this theorem does not guarantee the reconstruction of the input also obtains the ideal state. Therefore, we develop our algorithm to learn robust feature representation by minimizing the reconstruction error of both the input and hidden representation.

\begin{myRemark}
It presents the relationship between the reconstruction error of the input and the reconstruction error of hidden representation as well as the reconstruction error of hidden representation and the objective function of CAE.
\end{myRemark}
\begin{myRemark}
This theorem is also the main evidence that minimizing reconstruction error of hidden representation is more robust for feature learning than minimizing the Frobenius norm of Jacobia matrix of hidden representation.
\end{myRemark}

Since the proposed DDAEs use the reconstruction of hidden representation as the objective function and CAEs learn features by minimizing the Frobenius norm of Jacobia matrix of hidden representation, we can conclude that DDAEs are more robust for feature representation than CAEs. This may also be the main reason why DDAEs always outperform CAEs.

\subsection{Lower Bound with Corrupted Input}

We now show that when the input is corrupted with noises, the reconstruction error of the corrupted input also has a lower bound.
\newtheorem{Theorem 2}{Theorem}
\begin{Theorem} \label{Theorem 2}
Let $L(x,y)=\|x-y\|^{2}$ be the squared error. If some noises are added to the original input $x$, then as $\bar{x} \rightarrow \tilde{x}$, we have
\begin{equation}
E L(\tilde{x},g(\tilde{h})) \geq  E  L(h,f(g(\tilde{h}))) /  E\|J_{f}(\tilde{x})\|_{F}^{2},
\end{equation}
where $\tilde{x}$ is the corrupted input, $h=f(\tilde{x})$ is the hidden representation, $\tilde{h}$ is the corrupted hidden representation and $\bar{x}=g(\tilde{h})$ is the intermediate reconstructed input.
\end{Theorem}
\begin{proof}
Let $h^{*}=f(\bar{x})$ be the reconstructed hidden representation. Then we can approximate the encoder function $f(\bar{x})$ by its Taylor expansion around $\tilde{x}$ with Lagrange remainder term
\begin{equation}
f(\bar{x})=f(\tilde{x})+(\bar{x}-\tilde{x})^{T}\nabla f[\tilde{x}+\rho(\bar{x}-\tilde{x})],    \nonumber
\end{equation}
where $\nabla f[\tilde{x}+\rho(\bar{x}-\tilde{x})]$ is the first-order derivative of function $f(\cdot)$ with respect to $\tilde{x}+\rho(\bar{x}-\tilde{x})$ and $\rho\in(0,1)$ is a constant. Here $\nabla f[\tilde{x}+\rho(\bar{x}-\tilde{x})]$ is a real-valued random matrix and $\tilde{x}+\rho(\bar{x}-\tilde{x})$ is a real-valued random vector.

With triangle inequality, we get that
\begin{flalign}
\begin{split}
E L(h,f(g(\tilde{h})))& = E \|h^{*}-h \|^{2}   \\ \nonumber
&= E \|f(\bar{x})-f(\tilde{x})\|^{2} \\  \nonumber
&= E \|(\bar{x}-\tilde{x})^{T} \nabla f[\tilde{x}+\rho(\bar{x}-\tilde{x})] \|^{2}\\ \nonumber
&\leq E\| \bar{x}-\tilde{x}\|^{2} \cdot E\| \nabla f[\tilde{x}+\rho(\bar{x}-\tilde{x})] \|_{F}^{2}, \nonumber
\end{split}
\end{flalign}
where $\|\cdot \|^{2}$ is the squared error and $\|\mathcal{H} \|_{F}^{2}$ is the square of Frobenius norm on random matrix $\mathcal{H}$.

When the intermediate reconstructed input $\bar{x}$ infinitely approaches $\tilde{x}$, we get
\begin{equation}
\lim_{\bar{x}\to \tilde{x}} \| \nabla  f[\tilde{x}+\rho(\bar{x}-\tilde{x})] \|_{F}^{2} = \|J_{f}(\tilde{x})\|_{F}^{2}, \nonumber
\end{equation}
and
\begin{equation}
E  L(h,f(g(\tilde{h}))) \leq E L(\tilde{x},g(\tilde{h})) \cdot E\|J_{f}(\tilde{x})\|_{F}^{2},  \nonumber
\end{equation}
Hence, we have
\[ E L(\tilde{x},g(\tilde{h})) \geq  E  L(h,f(g(\tilde{h}))) /  E\|J_{f}(\tilde{x})\|_{F}^{2}.    \qedhere \]
\end{proof}

\begin{myRemark}
Theorem 2 summarizes that even though the input is corrupted with noises, the reconstruction error of the corrupted input also can not be lower than a lower bound, which is the guiding principle for reconstructing the corrupted input. However, the lower bound of this situation is an expectation.
\end{myRemark}
\begin{myRemark}
This theorem is also the main evidence why minimizing reconstruction error of hidden representation is more robust for feature representation than minimizing the Frobenius norm of Jacobia matrix of hidden representation when confronted with corrupted input.
\end{myRemark}

\section{Robustness of Hidden Representation Reconstruction}
\label{Section4}

In this section, we theoretically prove that minimizing the Frobenius norm of the Jacobian matrix of the hidden representation has a deficiency and may result in a much worse local optimum value. We also show that minimizing reconstruction error of hidden representation for feature representation is more robust than minimizing the Frobenius norm of Jacobia matrix of hidden representation.

\begin{figure*}[htbp]
\begin{center}
\centerline{\includegraphics[ width=14cm]{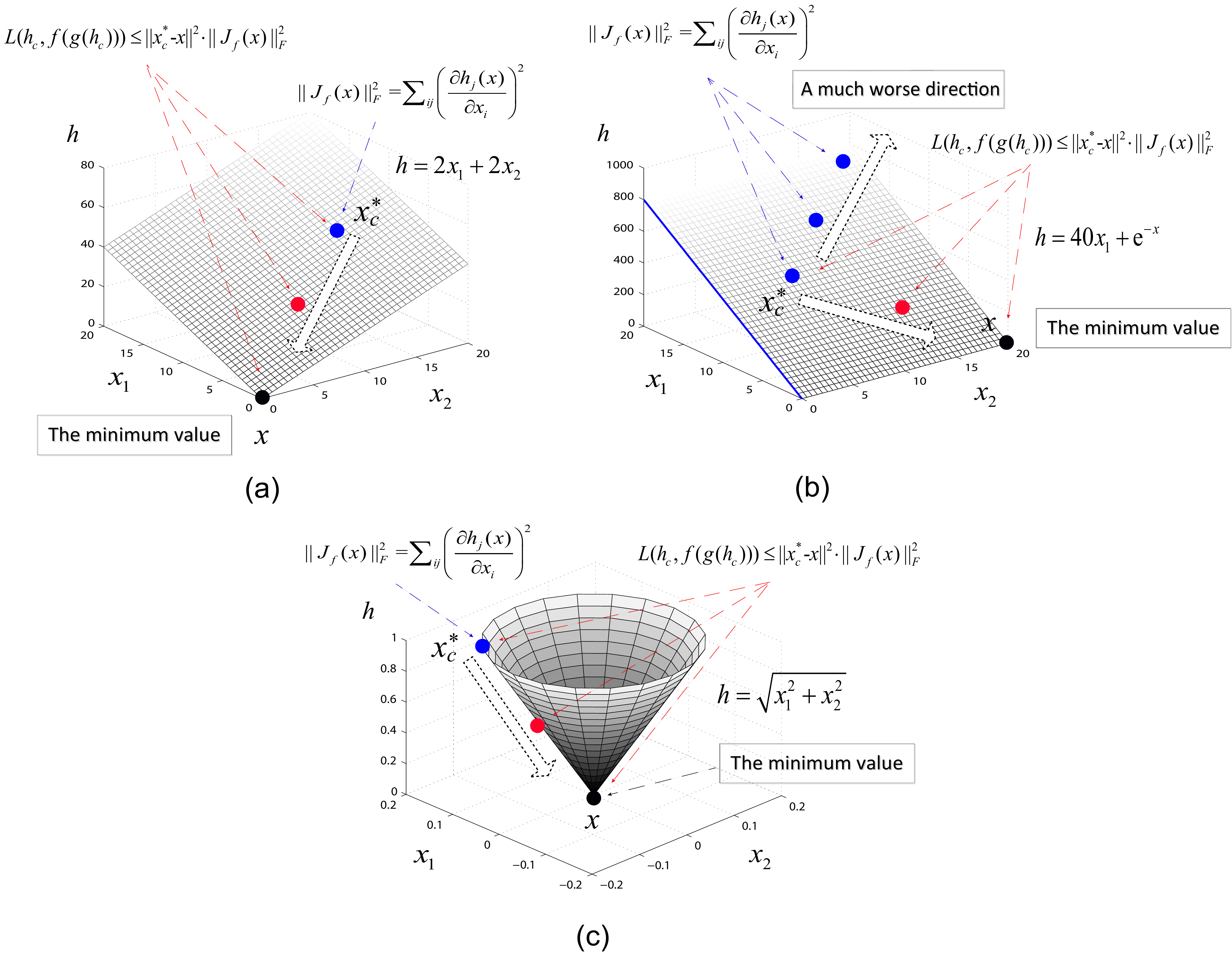}}
\caption{Three examples are provided to show the robustness of reconstruction of hidden representation in a three-dimensional space. (a) A plane that the formula is $h=2x_{1}+2x_{2}$, $x_{1}\geq 0$ and $x_{2}\geq 0$. When the algorithm hits the plane, the Frobenius norm of Jacobia matrix of hidden representation is a constant. Hence, minimizing the Frobenius norm of Jacobia matrix is invalid. It stops to search the optimum value or finds in a random direction. (b) A figure that the formula is $h=40x_{1}+e^{-x_{2}}$. For minimizing the Frobenius norm of Jacobia matrix of hidden representation, searching in a much worse direction is encouraged. (c) A cone that the minimum value is located at its base and any one of the first derivatives is not equal to a constant while the Frobenius norm of Jacobia matrix is a constant. For minimizing the Frobenius norm of Jacobia matrix of hidden representation, once it hits the edge of the cone, it stops finding the minimum value or searches in a random direction. However, for minimizing the reconstruction error of hidden representation, it may be propelled towards its base, the location of the minimum value.}
\label{Examples}
\end{center}
\end{figure*}

\subsection{Theoretical Explanation and Examples}

The main theoretical contribution of this paper is that we show when $x_{c}^{*} \rightarrow x$,
\begin{equation}
\begin{split}
L(h_{c},f(g(h_{c}))) & = \| (x_{c}^{*}-x)^{T}J_{f}(x)\|^{2} \\
& \leq  \| x_{c}^{*}-x\|^{2}  \cdot \|J_{f}(x)\|_{F}^{2},
\end{split}
\label{inequat}
\end{equation}
where $h_{c}$ is the corresponding hidden representation of the clean input $x$, i.e., $h_{c}=f(x)$, $x_{c}^{*}=g(h_{c})=g(f(x))$ is the reconstructed input, $f(\cdot)$ is the encoder function, $g(\cdot)$ is the decoder function, $L(h_{c},f(g(h_{c})))$ is the reconstruction error of hidden representation and $\|J_{f}(x)\|_{F}^{2}= \sum_{ij} \left( \frac{\partial h_{j}(x)}{\partial x_{i}} \right) ^{2} $ is the Frobenius norm of Jacobia matrix of hidden representation $h_{c}$ with respect to input $x$. We give the proof of the Inequation (\ref{inequat}) in Section \ref{Section3}. Now we theoretically show that minimizing the Frobenius norm of Jacobia matrix of hidden representation is invalid in some situations. Meanwhile, we also demonstrate that in these situations, reconstruction of hidden representation is more robust than minimizing the Frobenius norm of Jacobia matrix.

Let us consider three special optimization problems: 1) When the algorithm reaches such areas, all of the first derivatives are equal to constants; 2) Some of the first derivatives are equal to constants; 3) Any one of the first derivatives is not equal to a constant, but the Frobenius norm of Jacobia matrix is a constant.

\textit{Case 1.} We firstly consider the simple situation that all of the first derivatives are equal to constants, i.e., $\frac{\partial h_{j}(x)}{\partial x_{i}}=c_{ij}$ is a constant for all $i=1,2,...,D_{x}$ and $j=1,2,..,D_{h}$. In this situation, the Frobenius norm of Jacobia matrix of hidden representation is a constant. It means that once the algorithm reaches these areas, minimizing the Frobenius norm of Jacobia matrix losts its ability to find the optimum value. The algorithm stops early or searches in a random direction that even includes a much worse direction, far away from the optimum value. Therefore, minimizing this Frobenius norm is invalid for such a situation. However, when $x_{c}^{*} \rightarrow x$, because of Inequation (\ref{inequat}), the value of reconstruction error of hidden representation continues to decrease. Hence, when all of the first derivatives are equal to constants, minimizing reconstruction error of hidden representation works and continues to find the optimum value. It also means that reconstruction of hidden representation is more robust than minimizing the Frobenius norm of Jacobia matrix.

Example 1. When all of the first derivatives are equal to constants. In this situation, the solution space is in fact a plane in a three-dimensional space. Fig. \ref{Examples} (a) presents a plane that the formula is $h=2x_{1}+2x_{2}$, $x_{1}\geq 0$ and $x_{2}\geq 0$. When the algorithm hits the plane, the Frobenius norm of Jacobia matrix of hidden representation is a constant. Hence, minimizing the Frobenius norm of Jacobia matrix is invalid. It stops to search the optimum value or finds in a random direction. However, when $x_{c}^{*} \rightarrow x$, namely, the term of $\| x_{c}^{*}-x\|^{2}$ infinitely approaches 0, minimizing the reconstruction error of hidden representation may find towards its minimum value.

\textit{Case 2.} For the second situation, when some of the first derivatives are equal to constants, minimizing the Frobenius norm of Jacobia matrix of hidden representation seems to be working. However, we theoretically show that minimizing the Frobenius norm of its Jacobia matrix may encourage to obtain a much worse local optimum value. For this case, we only show the results that one of the first derivatives is equal to a constant. The similar results can be obtained with the multiple constants of the first derivatives.

\newtheorem{Theorem 3}{Theorem}
\begin{Theorem} \label{Theorem 3}
Let $(x^{u})^{T}=(x_{1}^{u},x_{2}^{u},...,x_{D_{x}}^{u})$ be the current value and the clean input, $x^{T}=(x_{1},x_{2},...,x_{D_{x}})$, be the optimum value. If one of the first derivatives is equal to constant, then there exists a next value $(x^{w})^{T}=(x_{1}^{w},x_{2}^{w},...,x_{D_{x}}^{w})$, obtained by minimizing the Frobenius norm of Jacobia matrix of hidden representation with respect to $x$, such that
\begin{equation}
\|x^{w}- x\|_{2}^{2}\gg \|x^{u}- x\|_{2}^{2},
\end{equation}
and
\begin{equation}
\|J_{f}(x^{w})\|_{F}^{2} < \|J_{f}(x^{u})\|_{F}^{2}.
\end{equation}
\end{Theorem}
\begin{proof}
We assume without losing generality that $\frac{\partial h_{p}(x)}{\partial x_{q}}=c_{pq}$ is a constant and all other first derivatives $\frac{\partial h_{j}(x)}{\partial x_{i}}$ are varying, where $i=1,2,...,q-1,q+1,...,D_{x}$ and $j=1,2,...,p-1,p+1,...,D_{h}$. For convenience, we also assume that there are only two different places between the current value $x^{u}$ and the next value $x^{w}$: one is the first position and the other is the $q$-th position, i.e., $x_{i}^{u}=x_{i}^{w}$, $i=2,...,q-1,q+1,...,D_{x}$.

Because $\frac{\partial h_{p}(x)}{\partial x_{q}}=c_{pq}$ is a constant, the integral $\int  \frac{\partial h_{p}(x)}{\partial x_{q}} d x_{q}  = c_{pq} x_{q} + c_{0}$ is a line in a multi-dimensional space, where $c_{0}$ is the bias. If we keep the same directions for all other first derivatives, $\frac{\partial h_{j}(x)}{\partial x_{i}}$, $i=1,2,...,q-1,q+1,...,D_{x}$, $j=1,2,...,p-1,p+1,...,D_{h}$, then for the direction $\frac{\partial h_{p}(x)}{\partial x_{q}}$, taking any value on this line $c_{pq} x_{q} + c_{0}$ has no effect on the objective function of minimizing the Frobenius norm of Jacobia matrix of hidden representation. Hence, along this line, we can take the value of the $q$-th dimension of the next value such that the next value is far away from the optimum value $x$ and its projection on the $x_{q}$-axis is very large, i.e., $|x_{q}^{w}-x_{q}|$  is a very large value.

In addition, if we keep the same directions for all other first derivatives except the direction $\frac{\partial h_{1}(x)}{\partial x_{1}}$, then along this direction, we can decrease the objective function of minimizing the Frobenius norm of Jacobia matrix of hidden representation and we can also get that $| \frac{\partial h_{1}(x^{w})}{\partial x_{1}} | < | \frac{\partial h_{1}(x^{u})}{\partial x_{1}}|$ and $|x_{1}^{w}-x_{1}|$ are bounded. Note that $|x_{1}^{u}-x_{1}|$ and $|x_{q}^{u}-x_{q}|$ are also bounded and only the first dimension and the $q$-th dimension are different. Therefore, we can get
\begin{equation}
\begin{split}
\|x^{w}- x\|_{2}^{2} &=(x_{1}^{w}-x_{1})^2+(x_{q}^{w}-x_{q})^2+ \mathop{ \sum_{i=2}} \limits_{i\neq q}^{D_{x}} (x_{i}^{w}-x_{i})^2  \\ \nonumber
&\gg (x_{1}^{u}-x_{1})^2+(x_{q}^{u}-x_{q})^2+ \mathop{ \sum_{i=2}} \limits_{i\neq q}^{D_{x}} (x_{i}^{u}-x_{i})^2 \\
&= \|x^{u}- x\|_{2}^{2}.
\end{split}
\end{equation}
and
\begin{equation}
\begin{split}
\|J_{f}(x^{w})\|_{F}^{2} &= \left( \frac{\partial h_{1}(x^{w})}{\partial x_{1}} \right) ^{2}  + \sum_{i=2}^{D_{x}} \sum_{j=2}^{D_{h}} \left( \frac{\partial h_{j}(x^{w})}{\partial x_{i}} \right) ^{2}  \\
&< \left( \frac{\partial h_{1}(x^{u})}{\partial x_{1}} \right) ^{2}  + \sum_{i=2}^{D_{x}} \sum_{j=2}^{D_{h}} \left( \frac{\partial h_{j}(x^{u})}{\partial x_{i}} \right) ^{2} \\ \nonumber
&=\|J_{f}(x^{u})\|_{F}^{2}.               \hfill \qedhere
\end{split}
\end{equation}
\end{proof}

\begin{myRemark}
Theorem 3 demonstrates that if one of the first derivatives is equal to a constant, minimizing the Frobenius norm of Jacobia matrix encourages to obtain a much worse local optimum value. However, because of Inequation (\ref{inequat}), minimizing the reconstruction error of hidden representation does work and may find the optimum value.
\end{myRemark}

Example 2. When some of the first derivatives are equal to constants. Fig. \ref{Examples} (b) demonstrates the figure that the formula is $h=40x_{1}+e^{-x_{2}}$. The minimum value is located at the bottom and one of the first derivatives is equal to a constant. Note that $\frac{\partial h(x)}{\partial x_{1}}$ is a constant, it has no contribution to minimizing the Frobenius norm. Hence, keeping the same direction for the other first derivative $\frac{\partial h(x)}{\partial x_{2}}$ and moving along this direction $\frac{\partial h(x)}{\partial x_{1}}$ is permissible, even if it is far away from the minimum value. The only limitation is that it should move in the direction of decreasing the value of $\frac{\partial h(x)}{\partial x_{2}}$. As a result, searching in a much worse direction is encouraged for minimizing the Frobenius norm of Jacobia matrix of hidden representation. Fig. \ref{Examples} (b) illustrates a much worse direction: the value of $\frac{\partial h(x)}{\partial x_{2}}$ decreases and the value of $\frac{\partial h(x)}{\partial x_{1}}$ is a constant, while the next value is far away from the minimum value, located at the bottom line. However, for minimizing the reconstruction error of hidden representation, as the term of $\| x_{c}^{*}-x\|^{2}$ infinitely approaches 0, it guarantees that the search direction is not away from the minimum value and it may move towards its bottom.

\textit{Case 3.} For the third situation, when all of the first derivatives are varying but the Frobenius norm of Jacobia matrix is a constant, it is similar to the first situation. More specifically, let $h_{j}(x)=f(x_{1},x_{2},...,D_{x})=\sqrt{x_{1}^{2}+x_{2}^{2}+...+x_{D_{x}}^{2}}$ be the encoder function for all $j=1,2,..,D_{h}$. Then we have that the Frobenius norm of Jacobia matrix is a constant, i.e., $\|J_{f}(x)\|_{F}^{2}= \sum_{ij} \left( \frac{\partial h_{j}(x)}{\partial x_{i}} \right) ^{2}= D_{x} \cdot D_{h}$. For such case, we also can prove that minimizing this Frobenius norm is invalid and reconstruction of hidden representation is more robust than minimizing the Frobenius norm of Jacobia matrix.

Example 3. When any one of the first derivatives is not equal to a constant, but the Frobenius norm of Jacobia matrix is a constant. In such an optimization problem, the solution space of this problem is in fact a cone in a three-dimensional space. Fig. \ref{Examples} (c) presents a cone, i.e., $h=\sqrt{x_{1}^{2}+x_{2}^{2}}$, where the minimum value is located at its base and any one of the first derivatives is not equal to a constant while the Frobenius norm of Jacobia matrix is a constant. For minimizing the Frobenius norm of Jacobia matrix of hidden representation, once it reaches the edge of the cone, it stops finding the minimum value or searches in a random direction. Nevertheless, for minimizing the reconstruction error of hidden representation, since the term of $\| x_{c}^{*}-x\|^{2}$ infinitely approaches 0 when $x_{c}^{*} \rightarrow x$, it is working and may be propelled towards its base, the location of the minimum value.

From the discussion above, we can conclude that when some or all of the first derivatives are equal to constants or all of the first derivatives are varying while the Frobenius norm of Jacobia matrix is a constant, minimizing the reconstruction error of hidden representation for feature representation is more robust than minimizing the Frobenius norm of Jacobia matrix. This may be the main reason why DDAEs always outperform CAEs in our experiments.

\section{Double Denoising Auto-Encoders}
\label{Section5}

We have shown that the necessary condition for the reconstruction of the input to reach the ideal state is that the reconstruction of hidden representation achieves its ideal condition in Section \ref{Section3}. We also show that minimizing (maximizing) the Frobenius norm of Jacobia matrix may get a much worse local optimum value and minimizing reconstruction error of hidden representation for feature representation is more robust than minimizing the Frobenius norm of Jacobia matrixas as illustrated in Section \ref{Section4}. Therefore, in this paper, we consider how to decrease the reconstruction error of hidden representation, which may get a better feature representation. We add the idea of reconstruction of hidden representation to the DAEs and propose a new deep learning model that takes the advantages of corruption and reconstruction. We anticipate that our proposed model has the capability to learn invariant and robust feature representation.

\begin{figure*}[htbp]
\begin{center}
\centerline{\includegraphics[width=14cm]{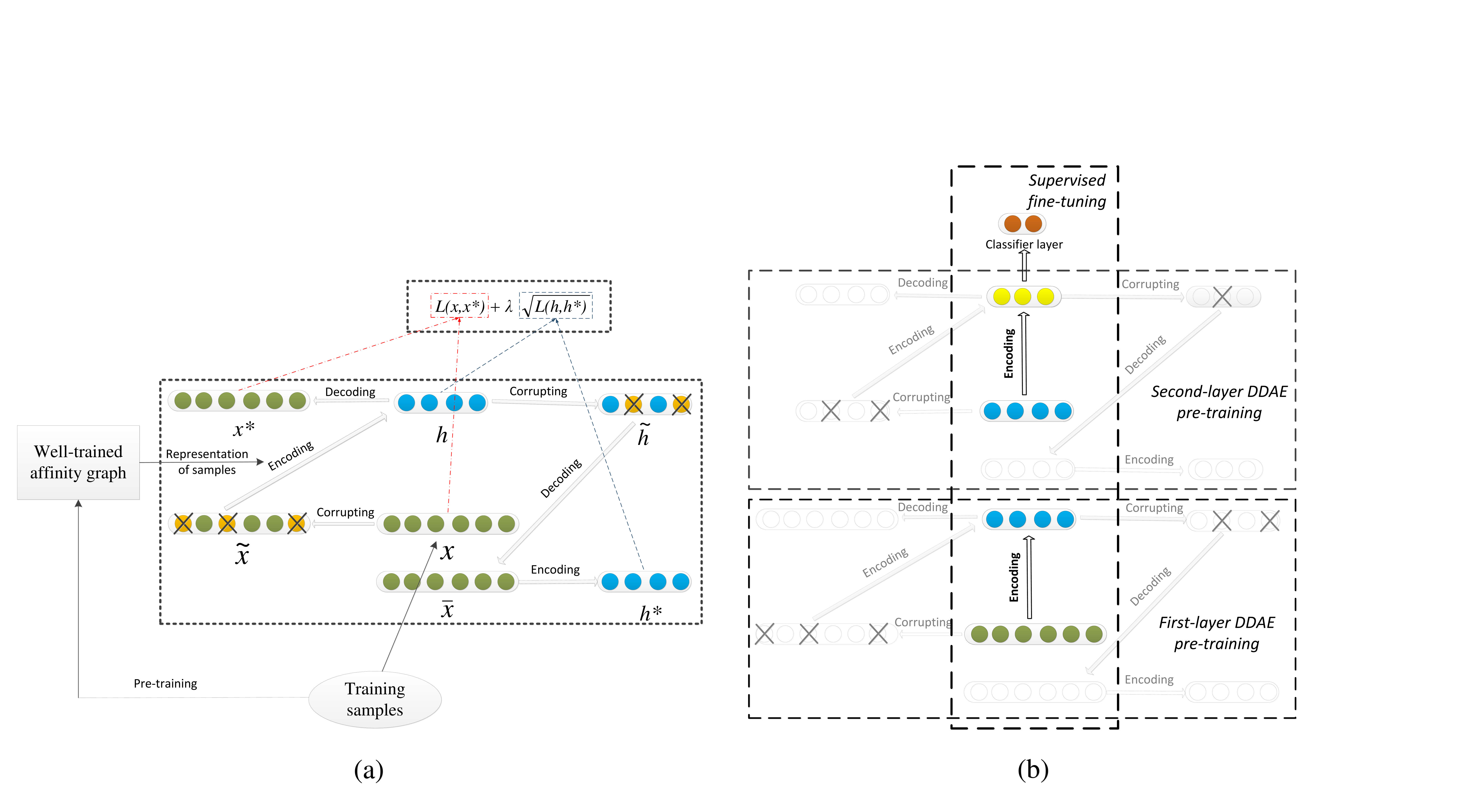}}
\caption{(a) The DDAE architecture. A sample $x$ is stochastically corrupted to $\tilde{x}$. The auto-encoder then maps it to hidden representation $h$ (via Encoding) and attempts to reconstruct $x$ via Decoding, producing reconstruction $x^{*}$. Reconstruction error is measured by the loss $L(x, x^{*})$. Meanwhile, the hidden representation $h$ is also stochastically corrupted to $\tilde{h}$, and then, $\tilde{h}$ is mapped to an intermediate reconstructed input $\bar{x}$ (via Decoding) and attempts to reconstruct $h$ via Encoding, producing reconstruction $h^{*}$. Reconstruction error is measured by the loss $\sqrt{L(h, h^{*})}$. (b) An example of two-layer DDAEs. Hidden representation of the first-layer DDAE is taken as the input of the second-layer DDAE. A classifier layer is added on the top of the stacked two-layer DDAEs to form a multilayer perceptron (MLP) classifier. For training a MLP classifier, the stacked two-layer DDAEs are firstly pre-trained in a greedy, layer-wise manner. After that, the MLP classifier, which is initialized by the pre-trained parameters, is fine-tuned by utilizing back-propagation.}
\label{DDAE}
\end{center}
\end{figure*}

\subsection{DDAEs Architecture}

As previously stated, a DDAE usually has two separate parts: constraints on the input (Constraints Part) and reconstruction on the hidden representation (Reconstruction Part). We use the DAE as the Constraints Part in a DDAE. In fact, one can replace Constraints
Part by any other auto-encoder variant. For example, we can replace the DAE with a Sparse Auto-Encoder or a CAE. It means that DDAE is flexible. The Reconstruction Part is done by first corrupting the hidden representation $h\in\Re^{D_{h}}$ into $\tilde{h}\in\Re^{D_{h}}$ according to a conditional distribution $q(\tilde{h}|h)$, and then mapping the corrupted hidden representation $\tilde{h}$ into an intermediate reconstructed input $\bar{x}=g(\tilde{h})=S_{g}(\textbf{W}'\tilde{h}+b_{x})\in\Re^{D_{x}}$ from which we reconstruct the hidden representation $h^{*}=f(\bar{x})=S_{f}(\textbf{W}\bar{x}+b_{h})$. Fig. \ref{DDAE} (a) illustrates a schematic representation of the procedure. Note that we use the reconstruction error of hidden representation $L(h,h^{*})=L(h,f(g(\tilde{h})))$ instead of the error between intermediate reconstructed input and original input, $L(x,\bar{x})$, or more complicated expressions, such as the combination of $L(x,\bar{x})$ and $L(h,h^{*})$. It is because the intermediate reconstructed input $\bar{x}=g(\tilde{h})=S_{g}(\textbf{W}'\tilde{h}+b_{x})$ is almost equal to the reconstruction input $x^{*}=g(h)=S_{g}(\textbf{W}'h+b_{x})$. As a result, $L(x,\bar{x})$ has the similar effect on the Constraints Part $L(x,x^{*})=L(x,g(f(\tilde{x})))$.

Fig. \ref{DDAE} (b) demonstrates an example of two-layer DDAEs. Usually, a DDAE is used to stack multiple layers to form a deep DDAEs architecture: output of a DDAE is used as input of the next DDAE. A classifier layer is built on the top of the stacked deep DDAEs architecture to form a multi-layer classifier. For training a multi-layer classifier, the stacked deep DDAEs architecture is firstly pre-trained in a greedy, layer-wise manner. Subsequently, the multi-layer classifier is initialized by the pre-trained parameters and fine-tuned by utilizing back-propagation.

\subsection{Training DDAEs}

To train a DDAE, there are two ways: one is to optimize a combination of Constraints Part and Reconstruction Part (DDAE-COM); the other is to optimize them separately (DDAE-SEP). For convenience, we use a linear combination of Constraints Part and Reconstruction Part as the objective function of DDAE-COM. Parameters $\theta=\{\textbf{W}, b_{x}, b_{h}\}$ are trained to minimize the reconstruction error over a training set $X=\{x_{1}, x_{2}, \cdots, x_{N}\}$. The objective function optimized by stochastic gradient descent becomes:
\begin{equation}
\mathcal{J}_{COM}(\theta)=  \sum_{x\in X}E\left[ L(x,g(f(\tilde{x})))+\lambda\sqrt{L(h,f(g(\tilde{h})))}\right],
\label{Equt:15}
\end{equation}
where $L(x,g(f(\tilde{x})))$ is the reconstruction error of the DAE (Constraints Part), $\sqrt{L(h,f(g(\tilde{h})))}$ is the reconstruction error of hidden representation (Reconstruction Part), $E(\delta)$ is the mathematical expectation of $\delta$, $\tilde{x}\in \Re^{D_{x}}$ is obtained from a conditional distribution $q(\tilde{x}|x)$, $h=f(\tilde{x})$ and $\lambda$ is a hyper parameter that controls the tradeoff between Constraints Part and Reconstruction Part.

From Equation (\ref{Equt:15}), we can conclude that a DDAE can be regarded as a general expression that extends the DAE. If the hyper parameter $\lambda$ in (\ref{Equt:15}) is set to be 0, a DDAE is the same as that of the DAE. That also means the DAE is a special case of our proposed method, i.e., DDAE is a generalization of the basic DAE algorithm. It should be pointed out that we utilize $\sqrt{L(h,f(g(\tilde{h})))}$ instead of $L(h,f(g(\tilde{h})))$ to calculate the reconstruction error of hidden representation. It is because the value of $L(h,f(g(\tilde{h})))$ is large at the beginning of training and we need to normalize it.

For the way of optimizing separately, DDAE-SEP firstly minimizes the following objective function over a mini-batch $X'=\{x_{1},x_{2},\cdots,x_{m}\}$:
\begin{equation}
O_{1}: \ \ \mathcal{J}_{SEP-1}(\theta)=\sum_{x\in X'} E \left[ L(x,g(f(\tilde{x}))) \right],
\end{equation}
where $x$ is a training sample selected from a mini-batch $X'$ and $L(x,g(f(\tilde{x})))$ is the reconstruction error of the DAE (Constraints Part) on the selected sample. Subsequently, DDAE-SEP updates the parameters $\theta$ optimized by the first objective function $O_{1}$ and minimizes the second objective function:
\begin{equation}
O_{2}: \ \ \mathcal{J}_{SEP-2}(\theta)=\sum_{h\in H'}  E \left[ L(h,f(g(\tilde{h}))) \right].
\label{Equt:17}
\end{equation}
where $h$ is the corresponding hidden representation of $x$ with updated parameters $\theta$, $H'$ is the corresponding hidden representation of $X'$ and $L(h,f(g(\tilde{h})))$ is the reconstruction error of hidden representation (Reconstruction Part). Once the parameters $\theta$ are updated by objective functions $O_{1}$ and $O_{2}$, DDAE-SEP will train on the next mini-batch and repeat the same procedure until stopping criteria are met. For more details about how to implement DDAE-COM and DDAE-SEP, please refer to Algorithms 1 and 2.

Let $h=S_{f}(\textbf{W}\tilde{x}+b_{h})$ be the hidden representation. With linear+sigmoid mapping, the computational complexity of reconstruction error of the input (e.g. squared error $L(x,g(f(\tilde{x})))=\|-x+b_{x}+\sum_{j=1}^{D_{h}} h_{j} W_{j} \|^{2}$)
is $O\left(D_{x}^{2} \times D_{h}^{2}\right)$. Based on (\ref{Equt:15}),  we can see that the computation complexity of $\sqrt{L(h,f(g(\tilde{h})))}= \left | -h+b_{h}+\sum_{i=1}^{D_{x}} x_{i} W_{i} \right |$ is $O\left(D_{x} \times D_{h}\right)$.
From (\ref{Equt:17}), the computation complexity of objective function $O_{2}$ is $O\left(D_{x}^{2} \times D_{h}^{2}\right)$.
Therefore, both Algorithms (DDAE-COM and DDAE-SEP) have the same overall computational complexity of $O\left(D_{x}^{2} \times D_{h}^{2}\right)$.

\begin{algorithm}[t]
\SetAlgoNoLine
\KwIn{Training data $X$, learning rate $\eta$, mini-batch size $m$, tradeoff coefficient $\lambda$.}
Initialize the parameters $\theta=\{\textbf{W}, b_{x}, b_{h}\}$\;
\Repeat{Stopping criteria are met}{
   Select $m$ samples $X'$ from $X$\;
   Let $\mathcal{J}(\theta)= \sum\limits_{x\in X'} E\left[ L(x,g(f(\tilde{x})))+\lambda\sqrt{L(h,f(g(\tilde{h})))}\right]$\;
   Update parameters $\theta$ by $\theta \longleftarrow \theta - \eta \frac{\partial \mathcal{J}(\theta)}{\partial \theta}$.
}
\caption{The DDAE-COM Algorithm}
\label{alg:exam}
\end{algorithm}

\subsection{Properties of DDAEs}

Please note that we use corrupted hidden representation $\tilde{h}$ instead of a clean hidden representation $h$ to reconstruct a clean hidden representation in Equations (\ref{Equt:15}) and (\ref{Equt:17}). There are two main reasons: 1) Although DDAEs use the manifold learning to extract robust features, we can not guarantee all the noises have been eliminated. They may propagate to hidden representation. 2) Even if all the noises have been eliminated, DDAEs may learn some inessential features such as backgrounds. The two are also the reasons why corrupting and reconstructing hidden representation for dealing with noises or some inessential features such as backgrounds is more robust than DAEs. In DAEs, corrupting the input and then reconstructing it makes DAEs can learn robust features. Our proposed model not only corrupts and reconstructs the input, but also does the same thing on hidden representation. For feature representation, corrupting hidden representation and then reconstructing it can partially reduce the negative effects such as the noises propagated by the input or some inessential features such as backgrounds, while DAEs do not deal with such noises or inessential features. Therefore, compared with DAEs, our proposed model is more robust to deal with the noises and some inessential features.

\begin{algorithm}[t]
\SetAlgoNoLine
\KwIn{Training data $X$, learning rate $\eta_{1}$ and $\eta_{2}$, mini-batch size $m$.}
Initialize the parameters $\theta=\{\textbf{W}, b_{x}, b_{h}\}$\;
\Repeat{Stopping criteria are met}{
        Select $m$ samples $X'$ from $X$\;
        Let $\mathcal{J}_{SEP-1}(\theta)= \sum\limits_{x\in X'} E\left[ L(x,g(f(\tilde{x})))\right]$\;
        Update parameters $\theta$ by $\theta \longleftarrow \theta - \eta_{1} \frac{\partial \mathcal{J}_{SEP-1}(\theta)}{\partial \theta}$\;
        Recalculate the hidden representation $H'$ of the selected $m$ samples with updated $\theta$\;
        Let $\mathcal{J}_{SEP-2}(\theta)= \sum\limits_{h\in H'} E\left[ L(h,f(g(\tilde{h}))\right]$\;
        Update parameters $\theta$ by $\theta \longleftarrow \theta - \eta_{2} \frac{\partial \mathcal{J}_{SEP-2}(\theta)}{\partial \theta}$.
      }
\caption{The DDAE-SEP Algorithm}
\label{alg:example222}
\end{algorithm}

\section{Experiments}
\label{Section6}

We evaluate DDAEs on twelve UCI datasets, thirteen image recognition datasets and two human genome sequence datasets and compare the performance with competitive state-of-the-art models. Several important parameters will also be experimentally evaluated. All the experiments are tested
on a laptop with Intel-i7 2.4G CPU, 16G DDR3 RAM, Windows 10 and Python 2.7.

\subsection{A Description of Datasets}

The twelve UCI datasets are selected from the UCI machine learning repository to evaluate the performance of DDAEs with other algorithms. For all the UCI datasets, we utilize the 10-fold cross validation to evaluate the competing algorithms and give the average error rates with 10 runs. Note that most of the UCI datasets are tested in the recent work, Deep Support Vector Machine (DeepSVM) \cite{DeepSVM2016}. Table \ref{UCI:data} summarizes the basic information of twelve UCI datasets.

\begin{table}[htbp]
\caption{UCI datasets used in the experiments}
\label{UCI:data}
\begin{center}
\begin{tabular}{llll}
\hline
Dataset & Samples & Features & Classes \\
\hline
 $\textit{sonar}$      & 208  & 60 & 2 \\
 $\textit{ionosphere}$ & 351 & 34 & 2 \\
 $\textit{ILPD}$ & 583 & 10 & 2 \\
 $\textit{breast\_cancer}$ & 683 & 10 & 2 \\
 $\textit{australian}$ & 690 & 14 & 2 \\
 $\textit{diabetes}$ & 768 & 8 & 2 \\
 $\textit{vehicle}$ & 846 & 18 & 4 \\
 $\textit{vowel}$ & 990 & 10 & 11 \\
 $\textit{german\_numer}$ & 1000 & 24 & 2 \\
 $\textit{cardiotocography}$ & 2126& 21 & 10 \\
 $\textit{segment}$ & 2310 & 19 & 7 \\
 $\textit{splice}$ &3175 & 60 & 2 \\

\hline
\end{tabular}
\end{center}
\end{table}

The thirteen image recognition datasets consist of the well-known MNIST digits classification problem, eight benchmark datasets and four more complex image recognition datasets. The MNIST digits come from the 28$\times$28 gray-scale images of handwritten digits. The eight benchmark datasets consist of five ten-class problems modified from MNIST digits and three two-class problems with shape classification. The five ten-class problems are variants of MNIST digits: smaller subset of MNIST (\textit{basic}), digits with random angle rotation (\textit{rot}), digits with random noise background (\textit{bg-rand}), digits with random image background (\textit{bg-img}) and digits with rotation and image background (\textit{bg-img-rot}). The three two-class problems are shape classification tasks: white tall and wide rectangles on black background (\textit{rect}), tall and wide rectangular image overlayed on different background images (\textit{rect-img}), convex and concave shape (\textit{convex}). All these data sets are also used in the works of Larochelle et al. \cite{larochelle2007empirical}, Rifai et al. \cite{rifai2011contractive} and Vincent et al. \cite{vincent2010stacked} and divided into three parts: a training set for pre-training and fine-tuning, a validation set for the choice of hyper-parameters and a testing set for the report result. The four more complex image recognition datasets are NORB \cite{RFN2015}, CIFAR-10 \cite{RFN2015}, COIL-100 \cite{AAAI2017} and Caltech-101 \cite{Du2016Stacked}. Details on image recognition datasets are listed in Table \ref{Digit:data}.

\begin{table}[htbp]
\caption{Digit image recognition used in the experiments}
\label{Digit:data}
\vskip 0.15in
\begin{center}
\begin{small}
\begin{tabular}{lllll}
\hline
Dataset & Train & Valid. & Test & Classes\\
\hline

\textit{rect}     &1000   &200   &50000   & 2\\
\textit{rect-img}   &10000  &2000  &50000   &2\\
\textit{convex}     &7000   &1000  &50000   &2\\
\hline
MNIST                & 50000 & 10000  & 10000  &10  \\
\textit{basic}      & 10000 & 2000 & 50000  &10   \\
\textit{rot}        & 10000 & 2000 & 50000  &10   \\
\textit{bg-rand}      & 10000 & 2000 & 50000  &10   \\
\textit{bg-img}       & 10000 & 2000 & 50000  &10   \\
\textit{bg-img-rot}   & 10000 & 2000 & 50000  &10   \\

\textit{NORB}      & 19300 & 5000 & 24300  & 5   \\
\textit{CIFAR-10}   & 45000 & 5000 & 10000  &10   \\
\textit{COIL-100}   & 1000 & 200 & 6000  &100   \\
\textit{Caltech-101}  & 3030 & 300 & 5814  &101   \\

\hline
\end{tabular}
\end{small}
\end{center}
\vskip -0.1in
\end{table}

The two human genome sequence datasets are the standard benchmark from fruitfly.org for predicting gene splicing sites on human genome sequences. The first dataset is the Acceptor locations containing 6,877 sequences with 90 features. The second data set is the Donor locations including 6,246 sequences with 15 features. The Acceptor dataset has 70bp in the intron (ending with AG) and 20bp of the following exon. The Donor dataset has 7bp of the exon and 8bp of the following intron (starting with GT). All these sequences consist of four letters (A, T, C and G). To use these two datasets, we need firstly to transform the four letters with four real numbers and then use these datasets to classify each sequence.

\subsection{Experimental Verification}

In Section \ref{Section3}, we have proved that if the reconstruction of a hidden representation does not reach its ideal situation, the reconstruction of the input can not obtain the ideal value. In Section \ref{Section4}, we have shown that minimizing the reconstruction error of hidden representation for feature representation is more robust than minimizing the Frobenius norm of Jacobia matrix of hidden representation. However, one question still needs to be solved: how to experimentally validate these points?

\subsubsection{Reconstruction of Hidden Representation vs. Reconstruction of the Input}

To show the robustness of the reconstruction of hidden representation against the reconstruction of the input, we evaluate the classification performance by only minimizing the Reconstruction Error of Hidden Representation (REHR) and compare the results with only minimizing the reconstruction error of the input, i.e., stacked AEs (SAE) \cite{rifai2011contractive}. Fig. \ref{FIG:only} shows the classification error rates of REHR on thirteen image recognition datasets with different layers. For the sake of fairness, the results of SAE are derived from the work of Vincent et al. \cite{vincent2010stacked}. As shown in Fig. \ref{FIG:only}, REHR with 3 layers almost gets the better results on all the thirteen image recognition datasets.

Fig. \ref{Fig:ALL} presents some example images with corresponding filters learned by the models of stacked AEs (SAEs) \cite{rifai2011contractive}, stacked DAEs \cite{vincent2010stacked}, stacked CAEs \cite{rifai2011contractive}, stacked REHR and stacked DDAEs. This figure shows features learned by the first layer of all the models on the \textit{rect} and \textit{bg-img-rot} datasets.

\begin{figure}[htbp]
\begin{center}
\centerline{\includegraphics[width=14cm]{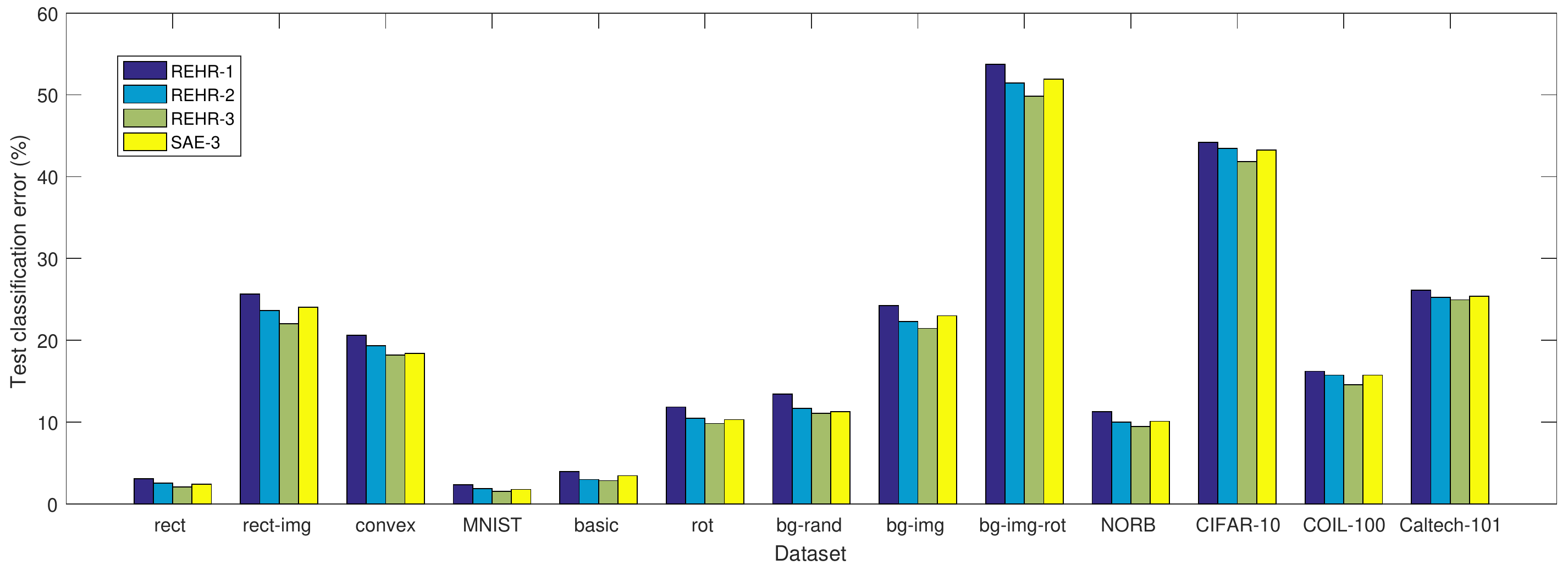}}
\caption{Classification error rates on thirteen benchmark classification tasks. The results are based on only minimizing the reconstruction error of hidden representation (REHR) with different layers. The results of SAE-3 are based on only minimizing the reconstruction error of the input and most results come from Vincent et al. \cite{vincent2010stacked}.}
\label{FIG:only}
\end{center}
\end{figure}

\begin{figure*}[htbp]
\begin{center}
\centerline{\includegraphics[width=14cm]{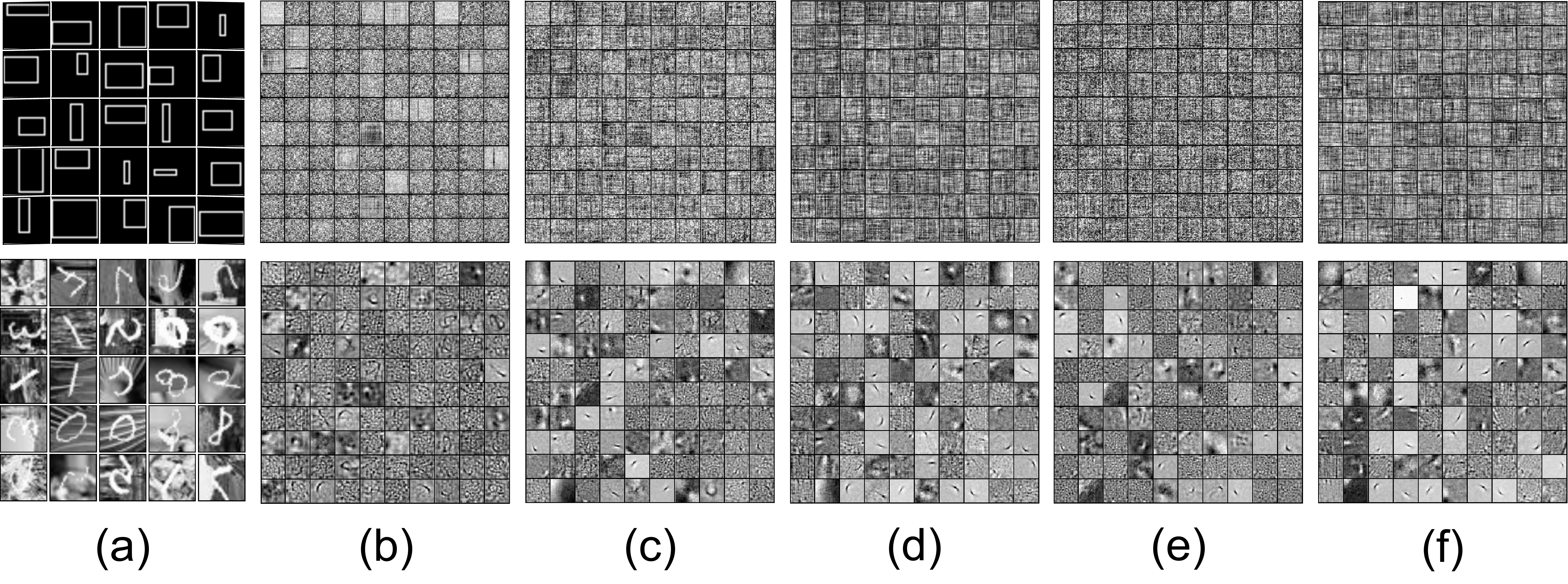}}
\caption{Example images with corresponding filters learned by different models on \textit{rect} (top) and \textit{bg-img-rot} (bottom) datasets. (a) Example images; (b) Filters learned by SAE; (c) Filters learned by CAE; (d) Filters learned by DAE; (e) Our results of REHR; (f) Our results of DDAE.}
\label{Fig:ALL}
\end{center}
\end{figure*}

\subsubsection{Reconstruction of Hidden Representation vs. Minimizing Frobenius Norm of Jacobia Matrix}

We conduct two comparison experiments to show the robustness of hidden representation reconstruction against minimizing the Frobenius norm of Jacobia matrix: 1) We compare the results of only using the reconstruction of hidden representation for feature representation with only minimizing the Frobenius norm of Jacobia matrix. 2) We also show the results of adding reconstruction of the input on both reconstruction of hidden representation and minimizing the Frobenius norm of Jacobia matrix. In the second situation, we only illustrate the comparison results of DDAEs and CAEs in practice.

In the first comparison experiment, we use MNIST as the testing dataset. As shown in Table \ref{MNIST:result}, we get the classification error rate of about 1.53\% with only minimizing the reconstruction error of hidden representation for feature representation. However, we can not get the classification error rate with only minimizing the Frobenius norm of Jacobia matrix of hidden representation. It is because the non-convergence problem will appear when we only use minimizing Frobenius norm of Jacobia matrix of hidden representation for feature representation. For the second comparison experiment, we just show the comparison results of DDAEs and CAEs. As we can see from the Tables \ref{result:UCI12} and \ref{Result:iamge}, DDAEs always outperform CAEs. Therefore, we can conclude that not only minimizing the reconstruction error of hidden representation for feature representation is more robust than minimizing the Frobenius norm of Jacobia matrix of hidden representation, but also DDAEs are more robust for feature representation than CAEs.

\subsection{Parameter Evaluation}

\begin{figure*}[htbp]
\begin{minipage}{0.48\linewidth}
  \centerline{\includegraphics[width=6.5cm] {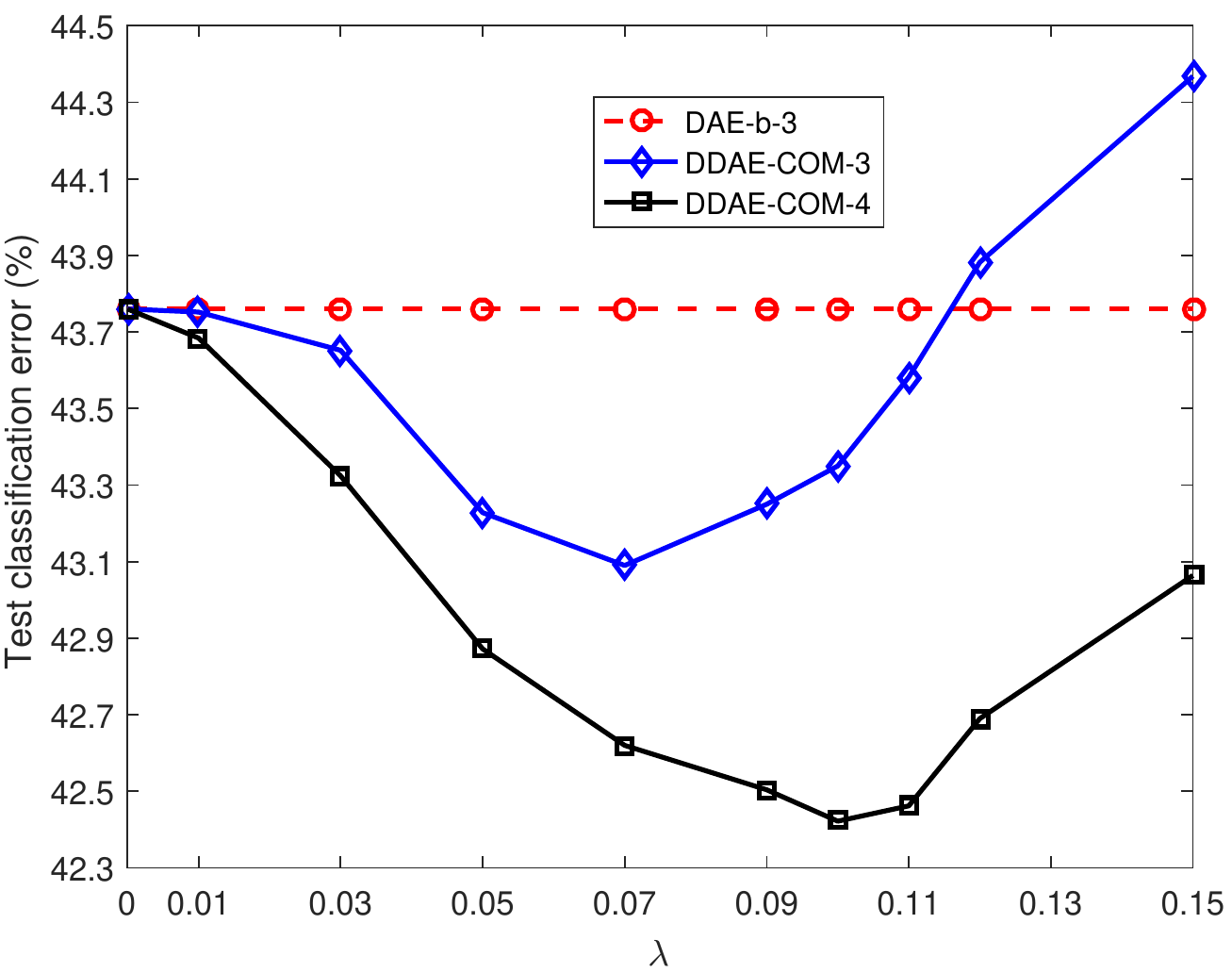}}
  \centerline{(a)}
\end{minipage}
\hfill
\begin{minipage}{.48\linewidth}
  \centerline{\includegraphics[width=6.5cm]{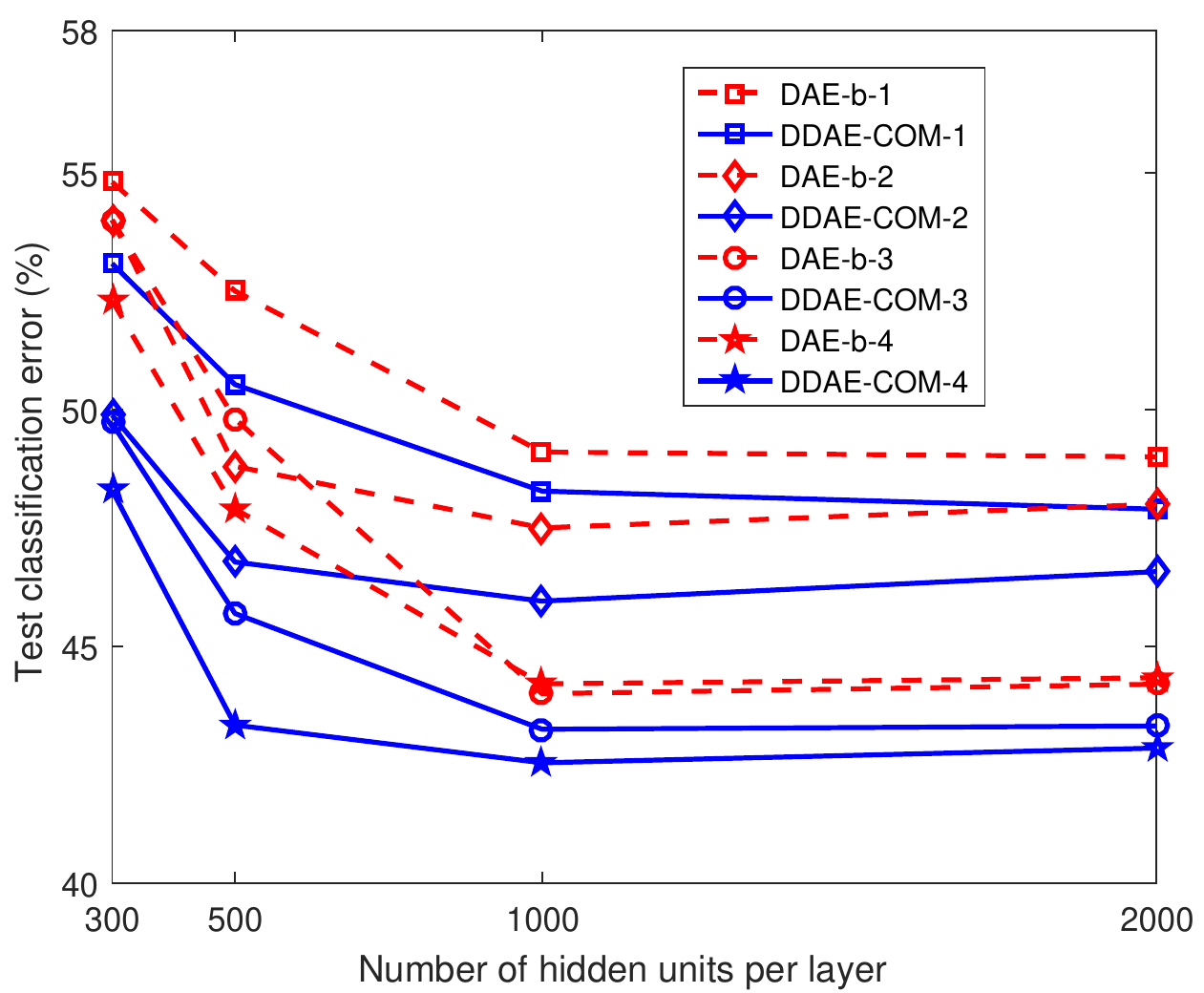}}
  \centerline{(b)}
\end{minipage}
\vfill
\begin{minipage}{0.48\linewidth}
  \centerline{\includegraphics[width=6.5cm]{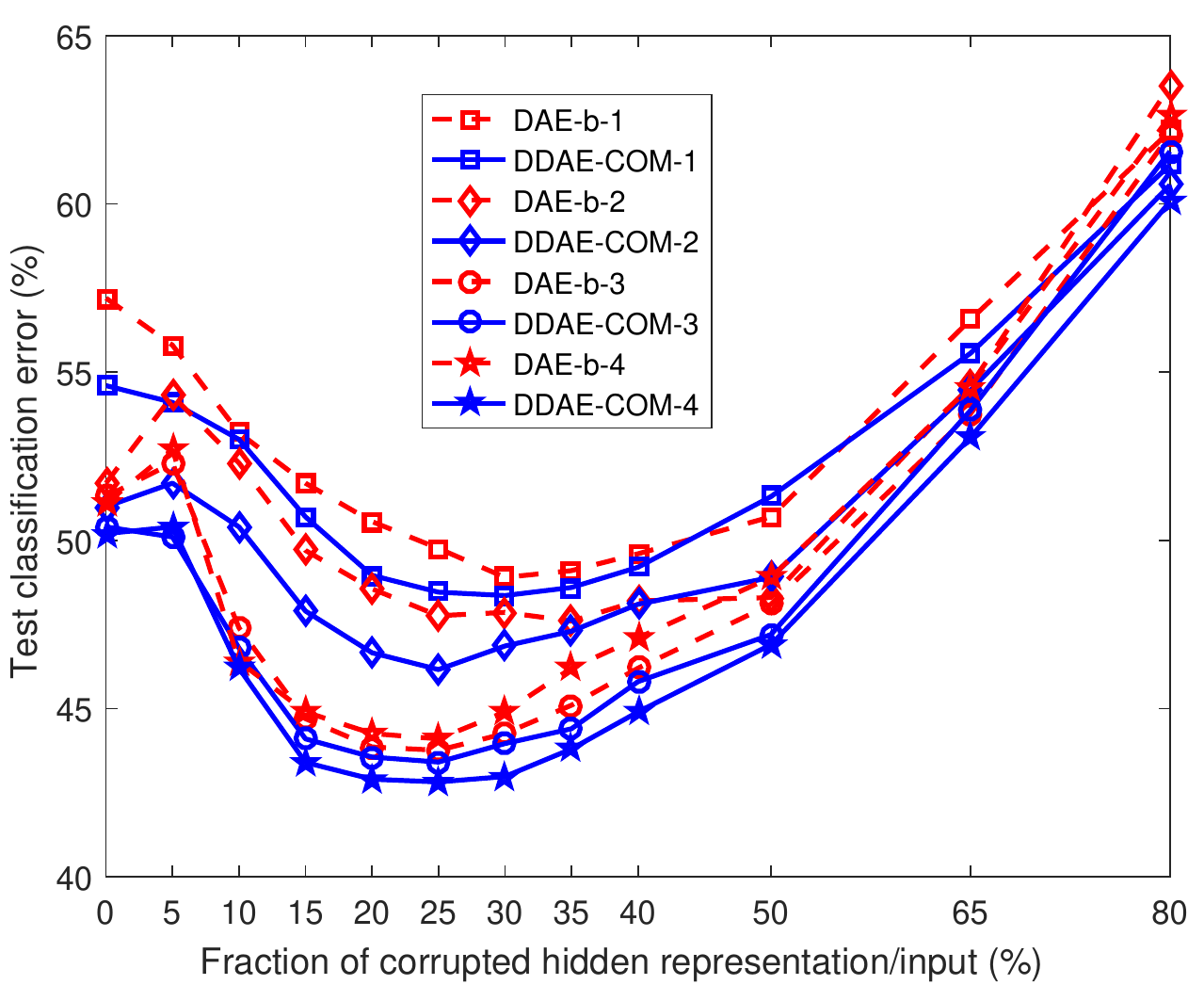}}
  \centerline{(c)}
\end{minipage}
\hfill
\begin{minipage}{0.48\linewidth}
  \centerline{\includegraphics[width=6.5cm]{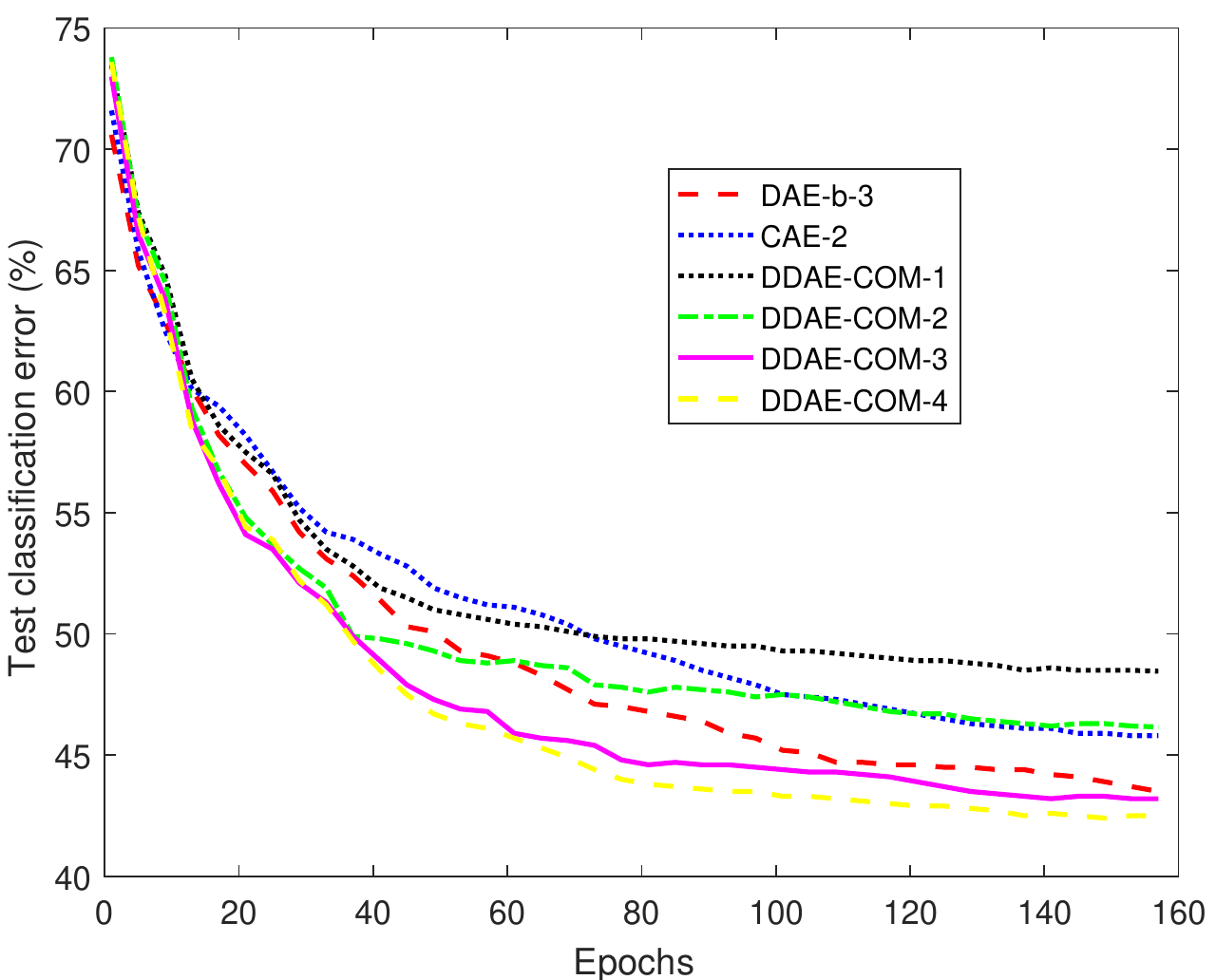}}
  \centerline{(d)}
\end{minipage}
\caption{Experimental results on \textit{bg-img-rot} dataset. (a) Classification error rates on the varying hyper-parameter $\lambda$ in DDAE-COM. DAE-b-3 is a 3 hidden layers stacked DAEs with masking noise and its result is obtained from Vincent et al. \cite{vincent2010stacked}. (b) A comparison of DDAEs and DAEs with increasing the number of hidden layers and the number of hidden units per layer. (c) Sensitivity to the fraction of corrupted hidden representation or input. (d) Test error rates with different training epochs.}
\label{Lamda:Compare}
\end{figure*}

In order to illustrate the effectiveness of hidden representation reconstruction, we evaluate the influence with the varying range of hyper parameters (the number of hidden layers, the number of units per hidden layer, the learning rate for unsupervised pre-training, the learning rate for supervised fine-tuning, etc.). In fact, it is difficult to find the optimal combination of the hyper parameters in a deep network. Fortunately, many researchers have proposed various rules for choosing hyper-parameters in the deep networks \cite{bergstra2011hyper}, \cite{bergstra2012random}, \cite{snoek2012practical}. In our experiments, we refer to the strategies used in \cite{larochelle2007empirical}. We initialize all the parameters with random values, and then fix other hyper parameters and perform a grid search over the range of one hyper parameter by utilizing mini-batch stochastic gradient descent.
%
To show the influence of hyper parameter $\lambda$ in DDAE-COM that controls the tradeoff between Constraints Part and Reconstruction Part, we compare DDAEs and DAEs with the adjustment of $\lambda$. For comparison, we use the \textit{bg-img-rot} as the testing dataset. We fix all other hyper parameters for both models and present the classification error rates on the \textit{bg-img-rot} dataset as shown in Fig. \ref{Lamda:Compare} (a).
Clearly, DDAEs perform better than DAEs when the hyper parameter $\lambda$ is located in a proper scope.

We also contrast DDAEs to DAEs with increasing the number of hidden layers and the number of hidden units per layer to show the influence of double corruption. Fig. \ref{Lamda:Compare} (b) shows the comparative classification error rates on \textit{bg-img-rot} dataset. The results of DDAEs marked blue in Fig. \ref{Lamda:Compare} (b) illustrate that as we increase the number of hidden layers from 1 to 4, the classification error rates gradually descend. It could be that DDAEs have the capability to capture the underlying data-generating distributions of both input and hidden representation while DAEs just capture the distribution of input. Fig. \ref{Lamda:Compare} (b) also shows that DDAEs outperform DAEs, especially when the number of hidden units per layer is low. It may be much easy for DDAEs to capture the underlying data-generating distribution of the hidden representation when the number of hidden representation per layer is low.

To assess the benefit of DDAEs on different corruption levels of the hidden representation or input, we compare the performance of DDAEs and DAEs with different numbers of hidden layers. Fig. \ref{Lamda:Compare} (c) demonstrates the sensitivity to corruption levels of hidden representation or input. We find that some corrupted hidden representation or input is beneficial. Fig. \ref{Lamda:Compare} (d) presents the relationships between the classification error rates and the training epochs. We can see that DDAEs
obtain the same performance of the DAEs or CAEs with much fewer training epochs.

\subsection{Comparisons with State-of-the-art Results}

To further show the robustness of hidden representation reconstruction for feature representation, we compare the performance with some state-of-the-art models on twelve UCI datasets, thirteen image recognition datasets and two human genome sequence datasets.

\subsubsection{UCI Dataset Classification}

We first test the classification performance of DDAEs on twelve UCI datasets. By utilizing DDAE-COM algorithm, we compare DDAEs with SVM model (SVM), stacked Deep Neural Networks (DNNs) \cite{DeepSVM2016}, stacked DeepSVMs \cite{DeepSVM2016}, stacked CAEs \cite{rifai2011contractive} and stacked DAEs \cite{vincent2010stacked}. We use tied weights, sigmoid activation function and squared error of reconstruction loss for the networks of CAEs, DAEs and DDAEs. In DDAEs model, we employ a 2-layer (200-150) architecture as the most frequently used structure. Because the number of features is small for a UCI dataset, we only use the corruption with first and second layers of DAEs and DDAEs, not with the input data.

Table \ref{result:UCI12} presents the classification error rates of 2 hidden layers stacked DDAEs with masking noise (DDAE-b-2), compared with SVM model (SVM), stacked Deep Neural Networks (DNNs), a 4 layers stacked DeepSVMs (DeepSVM-4), a 2 hidden layers stacked CAEs (CAE-2) and a 2 hidden layers stacked DAEs with masking noise (DAE-b-2). In general, networks with minimizing reconstruction error of hidden representation (DDAEs) perform better than other networks without this constraint. It is possible that DDAEs use the reconstruction of input and hidden representation, which may learn the underlying data-generating distributions of both input and hidden representation.

\begin{table}[htbp]
\caption{Classification error rates on UCI datasets with 10-fold cross validations. The best results obtained by all considered models are marked in bold.}
\label{result:UCI12}
\vskip 0.15in
\begin{center}
\begin{small}
\begin{tabular}{lrrrrrr}
\hline
Dataset & SVM & DNNs  & DeepSVM-4 & CAE-2 & DAE-b-2  & DDAE-b-2 \\
\hline
  \textit{sonar} & 14.67$\pm$\tiny{1.86}  & 13.31$\pm$\tiny{1.42}   & -       & 12.91$\pm$\tiny{1.46}  & 11.68$\pm$\tiny{1.41} & \textbf{11.23}$\pm$\tiny{\textbf{1.35}} \\
 \textit{ionosphere} & 12.70$\pm$\tiny{1.24}  & 12.14$\pm$\tiny{2.87} & 9.35$\pm$\tiny{1.44} &  8.87$\pm$\tiny{1.74} &  8.25$\pm$\tiny{1.72} & \textbf{7.86}$\pm$\tiny{\textbf{1.68}} \\
 \textit{ILPD} & 35.26$\pm$\tiny{1.79}    & 34.26$\pm$\tiny{1.67} & -        & 32.76$\pm$\tiny{1.68}  & 31.94$\pm$\tiny{1.58} & \textbf{31.28}$\pm$\tiny{\textbf{1.54}} \\
 \textit{breast\_cancer} & 2.93$\pm$\tiny{1.08}  & 0.79$\pm$\tiny{1.07} & 0.15$\pm$\tiny{3.11} &  0.12$\pm$\tiny{1.05} & 0.08$\pm$\tiny{1.02} & \textbf{0.05}$\pm$\tiny{\textbf{1.02}}\\
 \textit{australian} & 13.06$\pm$\tiny{1.56}  & 12.22$\pm$\tiny{1.75} & 11.02$\pm$\tiny{2.09} &  10.35$\pm$\tiny{1.48} & 10.43$\pm$\tiny{1.54} & \textbf{10.16}$\pm$\tiny{\textbf{1.43}} \\
\textit{diabetes} & 18.68$\pm$\tiny{1.22}   & 13.04$\pm$\tiny{1.78}   &12.48$\pm$\tiny{2.30}   &  11.96$\pm$\tiny{1.75} & 11.25$\pm$\tiny{1.63} & \textbf{10.84}$\pm$\tiny{\textbf{1.58}} \\
 \textit{vehicle} & 13.36$\pm$\tiny{1.25} & 12.85$\pm$\tiny{1.64}    & -       &  12.12$\pm$\tiny{1.65} & 11.83$\pm$\tiny{1.53} & \textbf{11.81}$\pm$\tiny{\textbf{1.47}} \\
\textit{vowel} & 1.83$\pm$\tiny{1.14}     & 1.52$\pm$\tiny{1.46}    & -         & 1.07$\pm$\tiny{1.57}  & 0.84$\pm$\tiny{1.52} &\textbf{0.47}$\pm$\tiny{\textbf{1.42}} \\
 \textit{german\_numer} & 22.40$\pm$\tiny{1.21}  & 15.88$\pm$\tiny{0.99} & 16.30$\pm$\tiny{1.33} &  16.35$\pm$\tiny{1.15} & 15.71$\pm$\tiny{1.12} & \textbf{15.28}$\pm$\tiny{\textbf{1.06}} \\
\textit{cardiotocography} & 21.84$\pm$\tiny{1.58}     & 20.25$\pm$\tiny{1.51}    & -       & 19.82$\pm$\tiny{1.47}  & 18.62$\pm$\tiny{1.35} &\textbf{18.26}$\pm$\tiny{\textbf{1.31}} \\
 \textit{segment} & 3.54$\pm$\tiny{1.26}     & 3.16$\pm$\tiny{1.34}       & -      & 3.07$\pm$\tiny{1.13} &  2.59$\pm$\tiny{1.05} & \textbf{2.17}$\pm$\tiny{\textbf{1.03}} \\
\textit{splice} & 16.66$\pm$\tiny{0.76}   & 7.57$\pm$\tiny{1.92}   & 6.91$\pm$\tiny{2.33}    &  6.48$\pm$\tiny{1.52}  & 5.97$\pm$\tiny{1.42} &\textbf{5.32}$\pm$\tiny{\textbf{1.40}} \\
\hline
\end{tabular}
\end{small}
\end{center}
\vskip -0.1in
\end{table}

\subsubsection{Digit Image Recognition}

After testing on the small UCI dataset classification problem, we compare DDAEs against the several state-of-the-art models for unsupervised feature extraction: SVM models with RBF kernel (SVM$_{rbf}$), stacked Deep Belief Networks (DBNs), stacked Deep Boltzmann Machines (DBMs) \cite{salakhutdinov2009deep}, stacked AEs (SAEs) \cite{rifai2011contractive}, stacked DAEs \cite{vincent2010stacked}, stacked CAEs \cite{rifai2011contractive}, stacked Rectified Factor Networks (RFNs) \cite{RFN2015} and Ladder Networks \cite{ladder-networks2015,Ladder2-2016}. All these models also adopt tied weights, sigmoid activation function for both encoder and decoder, and cross-entropy reconstruction loss except DBNs, DBMs and RFNs. DBNs and DBMs optimize the parameters by using contrastive divergence, while RFNs use the expectation-maximization algorithm. Stochastic gradient descent is applied as the optimization method for all these models.

The classification results and training time of DDAEs with other models on MNIST dataset are listed in Table \ref{MNIST:result}. By using zero-masking corruption noises (MN) and DDAE-COM algorithm, DDAEs with 3 layers can achieve an error rate of about 1.35\%, while the traditional DAEs is about 1.57\%. When using Gaussian corruption noises (GS), the test error of 3-layer DDAEs reduces to 1.12\% with DDAE-COM algorithm and 1.08\% with DDAE-SEP algorithm. With the well-known trick technique, dropout \cite{srivastava2014dropout}, the test error of 3-layer DDAEs can further reduce to about 0.69\% when training with DDAE-COM algorithm and about 0.66\% when training with DDAE-SEP algorithm. In the experiments, a 3-layer (1000-1000-2000) and a 4-layer (1000-1000-2000-1000) architecture are tested. All the hyperparameters are selected according to the performance on the validation set. As for dropout, we use a fixed dropout rate 20\% for all the input layers and the hidden layers. A momentum, which increases from 0.5 to 0.9, is adopted to speed up learning. A fixed learning rate of 4.0 is used and no weight decay is utilized.

\begin{table}[htbp]
\caption{Test error (in \%) and training time (seconds) of different models on MNIST.}
\label{MNIST:result}
\begin{center}
\begin{tabular}{lll}
\hline
Methods & Error (\%) & Training times (s) \\
\hline
SAE  \cite{rifai2011contractive}                                        & 1.78  &  39458.42 \\
DAE + MN  \cite{rifai2011contractive}                                        & 1.57  &  40865.81 \\
REHR +  3 layers                                       & 1.53    &  47579.63\\
SVM$_{rbf}$   \cite{vincent2010stacked}                                    & 1.40  &  - \\
DDAE-COM + MN + 3 layers                          & 1.35   &  49754.38 \\
RFNs   \cite{RFN2015}                                                 & 1.27    & - \\
DAE + GS   \cite{rifai2011contractive}                                        & 1.18  &  41459.54 \\
DBN        \cite{srivastava2014dropout}                                      & 1.18   &  57863.67\\
CAE        \cite{rifai2011contractive}                                        & 1.14    &  128786.97 \\
DDAE-COM + GS + 3 layers                          & 1.12 &  50362.85  \\
DDAE-SEP + GS + 4 layers                          & 1.09   &  69493.87  \\
DDAE-SEP + GS + 3 layers                          & 1.08  &  51976.26  \\
DBM        \cite{srivastava2014dropout}                                        & 0.96  &  129763.45 \\
DBN + dropout finetuning   \cite{srivastava2014dropout}                       & 0.92   &  57863.67  \\
DBM + dropout finetuning   \cite{srivastava2014dropout}                       & 0.79   &  129763.45  \\
DDAE-COM + GS + 3 layers + dropout finetuning        & 0.69  &  50362.85 \\
DDAE-SEP + GS + 3 layers + dropout finetuning        &   \textbf{0.66}   &  51976.26  \\
\hline
\end{tabular}
\end{center}
\end{table}

\begin{table}[htbp]
\caption{Comparison of stacked DDAEs with other models. Test error rate on all considered datasets is reported together with a 95\% confidence interval. The best results obtained by all these models are marked in bold. With DDAE-COM algorithm, DDAEs appear to achieve superior or equivalent to the best other model in ten out of twelve data sets with 3 layers (DDAE-3) and eleven out of twelve data sets with 4 layers (DDAE-4).}
\label{Result:iamge}
\vskip 0.15in
\begin{center}
\begin{small}
\begin{tabular}{@{}l@{}rrrrrrr@{}r@{}}
\hline
Dataset & SVM$_{rbf}$ & DBN-3 & DAE-b-3 & CAE-2 & RFNs & DDAE-3 & DDAE-4 \\
\hline
\textit{rect}       & 2.15$\pm$\tiny{0.13} & 2.60$\pm$\tiny{0.14} & 1.99$\pm$\tiny{0.12} & 1.21$\pm$\tiny{0.10} & 0.63$\pm$\tiny{0.06} & 0.65$\pm$\tiny{0.06} & \textbf{0.56}$\pm$\tiny{\textbf{0.06}}\\
\textit{rect-img}   & 24.04$\pm$\tiny{0.37} & 22.50$\pm$\tiny{0.37} & 21.59$\pm$\tiny{0.36} & 21.54$\pm$\tiny{0.36} & 20.77$\pm$\tiny{0.36} & 20.68$\pm$\tiny{0.36}  & \textbf{20.56}$\pm$\tiny{\textbf{0.36}}\\
\textit{convex}     & 19.13$\pm$\tiny{0.34} & 18.63$\pm$\tiny{0.34} & 19.06$\pm$\tiny{0.34} & -              & 16.41$\pm$\tiny{0.32} & 16.24$\pm$\tiny{0.32} & \textbf{15.78}$\pm$\tiny{\textbf{0.31}}\\
\textit{basic}      & 3.03$\pm$\tiny{0.15} & 3.11$\pm$\tiny{0.15} & 2.84$\pm$\tiny{0.15} & 2.48$\pm$\tiny{0.14} & 2.66$\pm$\tiny{0.14} & \textbf{2.43}$\pm$\tiny{\textbf{0.14}} & 2.45$\pm$\tiny{0.14}\\
\textit{rot}        & 11.11$\pm$\tiny{0.28} & 10.30$\pm$\tiny{0.27} & 9.53$\pm$\tiny{0.26} & 9.66$\pm$\tiny{0.26} & -              & 9.15$\pm$\tiny{0.25} & \textbf{9.08}$\pm$\tiny{\textbf{0.25}}\\
\textit{bg-rand}    & 14.58$\pm$\tiny{0.31} & \textbf{6.73}$\pm$\tiny{\textbf{0.22}} & 10.30$\pm$\tiny{0.27} & 10.90$\pm$\tiny{0.27} & 7.94$\pm$\tiny{0.24} & 10.18$\pm$\tiny{0.27} & 10.19$\pm$\tiny{0.27}\\
\textit{bg-img}     & 22.61$\pm$\tiny{0.37} & 16.31$\pm$\tiny{0.32} & 16.68$\pm$\tiny{0.33} & 15.50$\pm$\tiny{0.32} & 15.66$\pm$\tiny{0.32} & \textbf{14.49}$\pm$\tiny{\textbf{0.31}} & 14.51$\pm$\tiny{0.31}\\
\textit{bg-img-rot } & 55.18$\pm$\tiny{0.44} & 47.39$\pm$\tiny{0.44} & 43.76$\pm$\tiny{0.43}& 45.23$\pm$\tiny{0.44} & -               & 43.41$\pm$\tiny{0.43} & \textbf{42.82}$\pm$\tiny{\textbf{0.43}}\\
\textit{NORB } & 11.60$\pm$\tiny{0.40} & -  & 9.50$\pm$\tiny{0.37}& 9.85$\pm$\tiny{0.42} & 7.00$\pm$\tiny{0.32}               & 6.98$\pm$\tiny{0.37} & \textbf{6.85}$\pm$\tiny{\textbf{0.36}}\\
\textit{CIFAR-10} & 62.70$\pm$\tiny{0.95} & 43.38$\pm$\tiny{0.97} & 43.76$\pm$\tiny{0.94} & 42.85$\pm$\tiny{0.94} & 41.29$\pm$\tiny{0.95}  & 40.71$\pm$\tiny{0.92} & \textbf{40.25}$\pm$\tiny{\textbf{0.92}}\\

\textit{COIL-100} & 18.68$\pm$\tiny{0.82} & 16.15$\pm$\tiny{0.78} & 15.42$\pm$\tiny{0.71} & 15.37$\pm$\tiny{0.74} & -
& 14.92$\pm$\tiny{0.68} & \textbf{14.36}$\pm$\tiny{\textbf{0.68}}\\

\textit{Caltech-101} & 28.93$\pm$\tiny{0.87} & 26.45$\pm$\tiny{0.85} & 25.27$\pm$\tiny{0.79} & 25.22$\pm$\tiny{0.81} &  -
& 24.68$\pm$\tiny{0.76} & \textbf{24.62}$\pm$\tiny{\textbf{0.76}}\\
\hline
\end{tabular}
\end{small}
\end{center}
\vskip -0.1in
\end{table}

By using DDAE-COM algorithm, we also test our new model on the ten deep learning benchmark datasets with 3 and 4 layers (DDAE-3 and DDAE-4). Table \ref{Result:iamge} reports the resulting classification performance for our model (DDAE-3 and DDAE-4), together with the performance of SVMs with RBF kernel, a 3 hidden layers stacked Deep Belief Networks (DBN-3), a 3 hidden layers stacked DAEs with masking noise (DAE-b-3), a 2 hidden layers stacked CAEs (CAE-2) and a stacked Rectified Factor Networks (RFNs). As we can see from the table, our new model works remarkably well on all datasets. It is better than or equivalent to the state-of-the-art models in 10 out of 12 datasets with three layers and 11 out of 12 datasets with four layers. It should be pointed out that we give the same corrupting noise level to both input and hidden representation in each experiment.

\subsubsection{Human Genome Sequence Classification}

In order to further demonstrate the effectiveness of DDAEs, we evaluate the performance on two human genome sequence datasets. For these two sequence datasets, we also use the 10-fold cross validation and present the average error rates with 10 runs. Fig. \ref{DNA} illustrates the performance of stacked SAEs (SAE-2), stacked REHR (REHR-2), stacked CAEs (CAE-2), stacked DAEs with masking noise (DAE-b-2) and  stacked DDAEs with masking noise (DDAE-b-2). Note that all the models use 2 hidden layers. With DDAE-b-2 model, we obtain the classification error rates of 8.36\% for Acceptor dataset and 9.85\% for Donor dataset, the best results of five compared models in these experiments.

\begin{figure}[htbp]
\begin{center}
\centerline{\includegraphics[width=7cm]{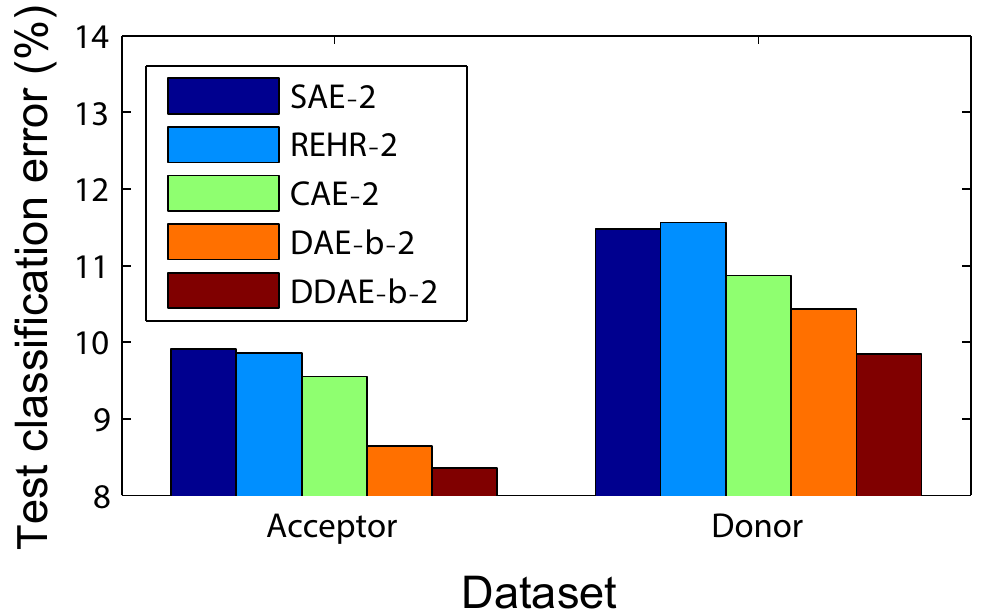}}
\caption{Experimental results on genome sequence dataset.}
\label{DNA}
\end{center}
\end{figure}

\section{Conclusion and Future Work}

In this paper, we demonstrated that the reconstruction error of the input has a lower bound and minimizing the Frobenius norm of Jacobia matrix of hidden representation has a deficiency and may encourage to get a much worse local optimum value. Based on these evidences, a new deep neural network, DDAEs, for unsupervised representation learning was proposed by using the idea of learning invariant and robust features for the small change on both input and hidden representation. The idea was implemented by minimizing the reconstruction error after injecting noises into both input and hidden representation. It is shown that our model is flexible and extendible. It is also demonstrated that minimizing the reconstruction error of hidden representation for feature representation is more robust than minimizing the Frobenius norm of Jacobia matrix of hidden representation. The comprehensive experiments indicated that DDAEs can achieve the performance superior to the existing state-of-the-art models on twelve UCI datasets and two human genome sequence datasets. For the thirteen image recognition datasets, DDAEs were better than or equivalent to most state-of-the-art models. Since the reconstruction of the hidden representation always helps an auto-encoder to perform better, and competes or improves upon the representations learning, how to design a useful constraint or any other operations on the hidden representation for the development of a more efficient representation learning model would be an interesting extension in further studies.

\bibliographystyle{ACM-Reference-Format}
\bibliography{ddae}

\end{document}